%% file: camera_ready.tex
\documentclass[runningheads]{llncs}

\usepackage{eccv}

\usepackage{eccvabbrv}

\usepackage{graphicx}
\usepackage{booktabs}

\usepackage[accsupp]{axessibility}

\usepackage{hyperref}

\usepackage{orcidlink}

\input{preamble}

\input{math_commands.tex}

\usepackage[utf8]{inputenc}
\usepackage[T1]{fontenc}
\usepackage{hyperref}
\usepackage{url}
\usepackage{booktabs}       
\usepackage{amsfonts}       
\usepackage{nicefrac}       
\usepackage{microtype}      
\usepackage{graphicx}
\usepackage{amsmath}
\usepackage{amssymb}
\usepackage{algorithm}
\usepackage{algorithmic}
\usepackage{multirow}
\usepackage{multicol}
\usepackage{bm}
\usepackage[flushleft]{threeparttable}
\usepackage{algorithm}      
\usepackage{setspace}
\usepackage{pifont}
\usepackage{subcaption}
\usepackage{graphicx}
\newcommand{\xmark}{\ding{55}}%
\usepackage{caption}
\usepackage{wrapfig}
\usepackage{arydshln}
\usepackage{bbm}
\usepackage{footmisc}
\usepackage{authblk}
\newcommand*\samethanks[1][\value{footnote}]{\footnotemark[#1]}

\begin{document}

\title{CLOSER: Towards Better Representation Learning for Few-Shot Class-Incremental Learning} 

\titlerunning{CLOSER: Representation Learning for Few-Shot Class-Incremental Learning}

\author{Junghun Oh\thanks{Equal contribution}\inst{1}\orcidlink{0009-0007-0840-9774} \and
Sungyong Baik\samethanks\inst{3}\orcidlink{0000-0001-5702-4618} \and
Kyoung Mu Lee\inst{1,2}\orcidlink{0000-0001-7210-1036}}

\authorrunning{J. Oh$^*$, S. Baik$^*$, and K.M. Lee}

\institute{$^1$Dept. of ECE\&ASRI, $^2$IPAI, Seoul National University\\
$^3$Dept. of Data Science, Hanyang University\\
\email{\{dh6dh,kyoungmu\}@snu.ac.kr},
\email{dsybaik@hanyang.ac.kr}}

\maketitle

\input{camera_ready/abstract}
\input{camera_ready/introduction}
\input{camera_ready/related_works}

\input{camera_ready/method}

\input{camera_ready/experiments}

\input{camera_ready/conclusion}

\noindent\textbf{Acknowledgments.}
This work was supported in part by the IITP grants [No. 2021-0-01343, Artificial Intelligence Graduate School Program (Seoul National University), No.2021-0-02068, and No.2023-0-00156], the NRF grant [No.2021M3A
9E4080782] funded by the Korean government (MSIT).

\newpage
\input{camera_ready_supple}

\newpage
\bibliographystyle{splncs04}
\bibliography{main}

\end{document}

%% file: preamble.tex
\usepackage[dvipsnames]{xcolor}


%% file: math_commands.tex
\newcommand{\gD}{\mathcal{D}}
\newcommand{\gC}{\mathcal{C}}
\newcommand{\gT}{\mathcal{T}}
\newcommand{\vx}{\boldsymbol{x}}
\newcommand{\vz}{\boldsymbol{z}}



%% file: camera_ready/abstract.tex
\begin{abstract}

Aiming to incrementally learn new classes with only few samples while preserving the knowledge of base (old) classes, few-shot class-incremental learning (FSCIL) faces several challenges, such as overfitting and catastrophic forgetting.
Such a challenging problem is often tackled by fixing a feature extractor trained on base classes to reduce the adverse effects of overfitting and forgetting.
Under such formulation, our primary focus is representation learning on base classes to tackle the unique challenge of FSCIL: simultaneously achieving the transferability and the discriminability of the learned representation.
Building upon the recent efforts for enhancing transferability, such as promoting the spread of features, we find that trying to secure the spread of features within a more confined feature space enables the learned representation to strike a better balance between transferability and discriminability.
Thus, in stark contrast to prior beliefs that the inter-class distance should be maximized, we claim that the closer different classes are, the better for FSCIL.
The empirical results and analysis from the perspective of information bottleneck theory justify our simple yet seemingly counter-intuitive representation learning method, raising research questions and suggesting alternative research directions.
The code is available \href{https://github.com/JungHunOh/CLOSER_ECCV2024}{here}.
\keywords{Few-shot class incremental learning \and Representation learning \and Transferability}

\end{abstract}

%% file: camera_ready/introduction.tex
\section{Introduction}
Owing to its strong representation power, deep neural networks (DNNs) boast outstanding performance across various fields.
However, such feats require tremendous human effort and time to collect an immense amount of data with accurate annotation.
The data hunger of DNNs poses a challenge, especially in dynamic real-world environments, where DNNs are required to learn new concepts with few examples while retaining previously learned concepts. 
To tackle the challenge, few-shot class-incremental learning (FSCIL)~\cite{tao2020topic} aims to design artificial intelligence systems that can learn new classes with few examples while maintaining performance on previously seen classes.

\input{camera_ready/Figure/main}

\footnotetext[1]{\scriptsize As in~\cite{liu2020negative,zheng2018ring,wang2018cosface}, we use an angular histogram to visualize 2D features of a DNN.}

To achieve the goal of FSCIL, we need to address catastrophic forgetting (forgetting of previous knowledge while learning new concepts)~\cite{delange2021continual,kirkpatrick2017ewc} and overfitting issues (overfitting to few examples, and thus poor generalization)~\cite{koch2015siamese,vinyals2016matching}.
To bypass this convoluted mixture of issues that hinder flexible adaptation of models, most previous works~\cite{zhang2021cec,hersche2022constrained,yang2023neural,zou2022margin,zhou2022forward,peng2022alice} fix the learned representation after training it on base (old) classes and employ a non-parametric classifier, using the feature-average class prototype representation~\cite{snell2017prototypical}.
However, such formulation leads to heavy reliance on the representation acquired through the optimization of softmax cross-entropy (SCE) loss on base classes, which often leads to collapsed intra-class representation~\cite{papyan2020prevalance} and poor transferability to new classes~\cite{islam2021broad}, as shown in Fig.~\ref{fig:main}a.
\textcolor{black}{Therefore, in this paper, we mainly focus on exploring effective representation learning methods, aiming to strike a better balance between discriminability on base classes and transferability to new classes.}

There \textcolor{black}{have} been great advances, particularly \textcolor{black}{self-supervised contrastive (SSC) learning}~\cite{chen2020simple,he2020momentum}, in the representation learning field to improve the transferability of the learned representation to downstream tasks.
Some works~\cite{islam2021broad,chen2022perfectly} have attributed the strong transferability of \textcolor{black}{SSC} learning to 'spread out' of intra-class features.
\textcolor{black}{Such representation places more emphasis on low- and mid-level features, which can be effectively transferred to and shared by new tasks.}
\textcolor{black}{Kornblith \textit{et al.}~\cite{kornblith2021why} have also found the relationship between the temperature of SCE loss and the spread of features, suggesting that lower temperature leads to better transferability.}
\textcolor{black}{As illustrated in Fig.~\ref{fig:main}b, we observe that the joint optimization of the SCE loss with low temperature and the SSC loss encourages the spread of features and better transferability to new classes.}

Despite the foregoing, \textcolor{black}{we observe that the previous methods for improving transferability are not enough to find a good representation for FSCIL; in fact, they harm the performance on base classes.}
Based on our experimental analysis, we find that excessive feature spread is very detrimental to base classes in the context of FSCIL because it hinders the feature-average class prototype from effectively representing its corresponding class, as demonstrated in Fig.~\ref{fig:main}b.
Hence, it's crucial to develop a representation learning method tailored specifically for the FCSIL problem, particularly addressing the unique challenge of simultaneously achieving discriminability on seen classes and transferability to unseen classes.

In this work, with the support from our experimental findings and information-bottleneck-theory-based analysis, we argue that the inter-class distance greatly affects the trade-off between discriminability and transferability of the learned representation in the FSCIL problem.
We find that the degraded discriminability due to the spread of features can be greatly regulated by reducing inter-class distance.
Moreover, we discover that reducing inter-class variability is also linked to enhancing the information bottleneck trade-off.
Thus, we claim that attempting to ensure the spread of features within a compressed representation space promotes learning minimal yet intrinsic task-related information.

Based on our analysis, in contrast to common beliefs and practices of previous FSCIL methods~\cite{yang2023neural,hersche2022constrained,zhou2022forward,song2023learning} that have attempted to increase the inter-class distance, we propose to \textit{decrease} it.
Incorporating SCE loss with lower temperature, SSC loss, and inter-class distance minimization, our new objective enables the learned representation to strike a better balance between discriminability and transferability, as illustrated in Fig.~\ref{fig:main}c.
With the simple yet seemingly counter-intuitive idea of bringing classes closer (hence the name \textbf{CLOSER}), the proposed method demonstrates outstanding performance, suggesting a new promising research avenue regarding representation learning for FSCIL.

%% file: camera_ready/Figure/main.tex
\begin{figure*}[t!]
    \centering
    \includegraphics[width=0.85\linewidth]{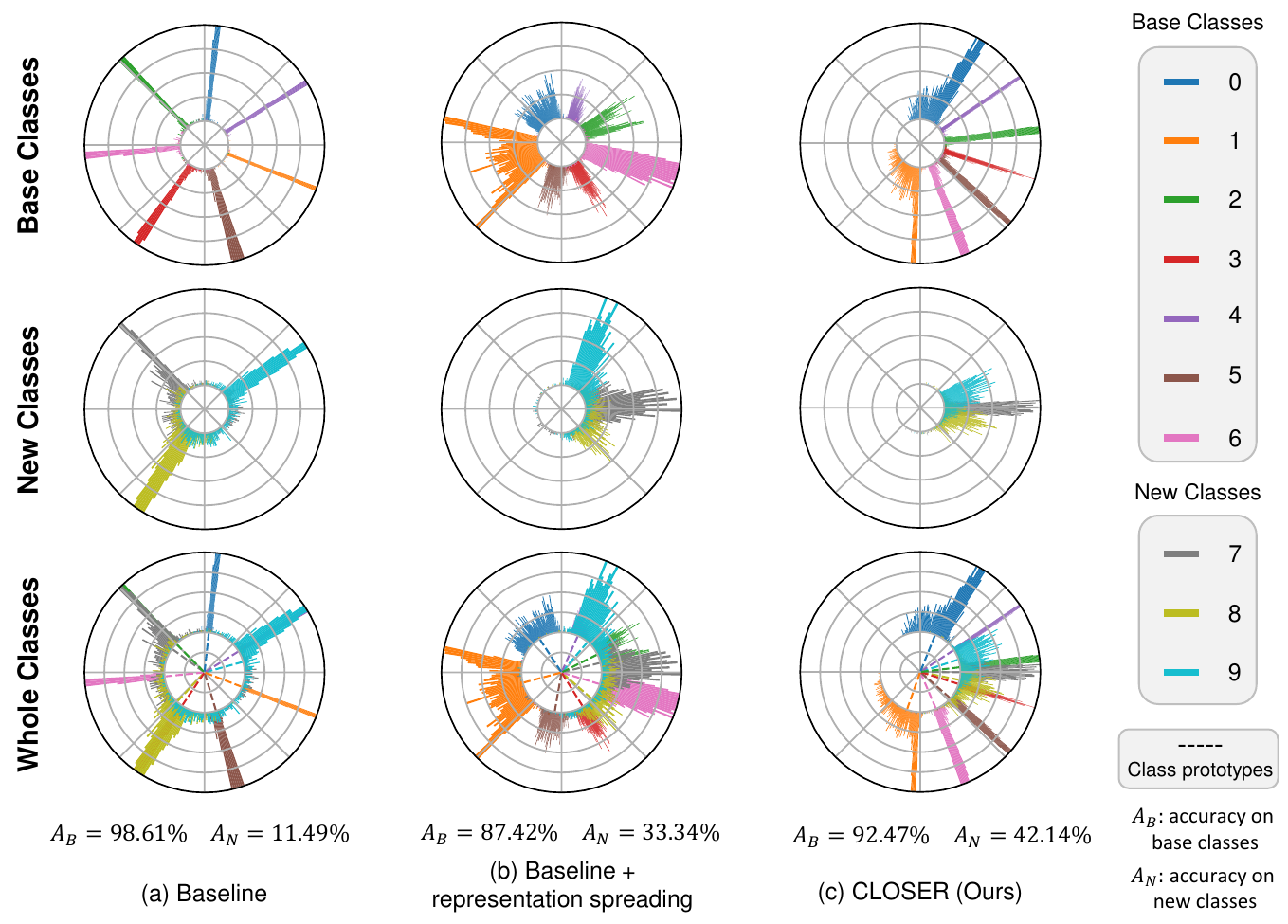}
    \caption{\textbf{Visualization of representation trained on MNIST\protect\footnotemark.}
    \textbf{(a) Baseline}~\cite{zhang2021cec,hersche2022constrained} exhibits great base-class discriminability (large inter-class distance) but weak transferability to the new classes \textcolor{black}{(huge overlap between new and base classes leading to misclassification)}.
    \textbf{(b) Baseline + representation spreading}~\cite{kornblith2021why,islam2021broad,chen2022perfectly} benefits the new classes (less collapse to the base classes), while compromising base-class discriminability in the context of FSCIL (\textcolor{black}{dispersed intra-class features leading to less accurate class representation with class prototypes}).
    \textbf{(c) CLOSER (Ours)}: 
    \textcolor{black}{Dispersing features in a narrowed feature space enhances both discriminability on the base classes (\textcolor{black}{less deviation between intra-class features and class prototypes}) and transferability to the new classes (even less overlap between the base and new classes).}
    \textcolor{black}{For instance, the \texttt{4} and \texttt{9} classes are not distinguishable in (b) and even less in (a), but CLOSER can yield representation that successfully discriminates them.}
    }
    \label{fig:main}
\end{figure*}

%% file: camera_ready/related_works.tex
\section{Related Works}
\noindent{\textbf{Few-Shot Class-Incremental Learning (FSCIL).}}
Towards the development of real-world artificial intelligence systems, Tao~\textit{et al.}~\cite{tao2020topic} have initially introduced few-shot class incremental learning, subsequently fostering numerous studies in the field~\cite{achituve2021gptree,cheraghian2021semantic,cheraghian2021synthesized,zhu2021self,shi2021overcoming,chen2021incremental,chi2022metafscil}.
Most of the works~\cite{peng2022alice,liu2022fewshot,yang2023neural,zhou2022fewshot,song2023learning,kalla2022s3c,zhuang2023gkeal,zhang2021cec,hersche2022constrained} bypass both catastrophic forgetting and overfitting issues by fixing the feature extractor trained on base classes and employing a non-parametric classifier using class prototypes~\cite{snell2017prototypical}.
Hence, recent works have focused on representation learning, where the common approach is to encourage greater separation between base classes to reserve the representation space for future new classes~\cite{yang2023neural,hersche2022constrained,zhou2022forward,song2023learning}.
However, we argue that increasing inter-class distance suppresses the acquisition of shared features among classes, which could be relevant to new classes.
Thus, contrary to the prevailing belief, we suggest learning representation with an inter-class distance minimization  and theoretically and empirically prove its effectiveness in improving the discriminability-transferability trade-off.
Similarly, Zou~\textit{et al.}~\cite{zou2022margin} emphasize the importance of learning shareable features among classes, which they propose to achieve with a negative margin~\cite{liu2020negative}.
A negative margin is more related to representation spreading rather than our idea of reducing inter-class distance, as discussed in Section~\ref{sec:supp:margintemp}.

\noindent{\textbf{Transferable Representation Learning.}}
The pursuit of representations that can be effectively transferred to downstream tasks has gained significant attention in recent years.
The early works have focused on supervised training on ImageNet dataset~\cite{ILSVRC15} and the way to transfer the knowledge to other tasks~\cite{chatfield2014return,razavian2014cnn,donahue2014decaf}.
Kornblith~\textit{et al.}~\cite{kornblith2021why} find that learned representations with better discriminability on a source dataset tend to show degraded transferability to downstream tasks.
To obtain a transferable representation, Liu~\textit{et al.}~\cite{liu2020negative,zou2022margin} suggest incorporating a negative margin in the softmax cross-entropy loss to promote feature sharing among classes rather than solely focusing on discriminative features.
In parallel, several methods have reported the strong transferability of representation yielded by self-supervised contrastive (SSC) learning~\cite{chen2020simple,he2020momentum}.
Islam~\textit{et al.}~\cite{islam2021broad} claim that the enhanced transferability acquired by SSC learning can be attributed to the spread of representation, implying that a network learns more fine-grained and shareable knowledge among tasks~\cite{chen2022perfectly}.
From the perspective of information bottleneck theory, Cui~\textit{et al.}~\cite{cui22discriminability} claim that \textit{over-compression} of mutual information between inputs and latent representations can prevent a network from learning features beneficial for downstream tasks.
In this paper, we argue that along with the representation spreading loss, directly regulating inter-class distance encourages a network to learn shareable features among classes, which could be advantageous for new classes.
Through the lens of information bottleneck theory, we theoretically analyze the connection between the joint optimization objective and the information bottleneck trade-off, supporting our claim.

%% file: camera_ready/method.tex
\section{Proposed Method}\label{method}
\subsection{Background: Problem Formulation}\label{problem}
Following the formulation of few-shot class incremental learning (FSCIL)~\cite{tao2020topic}, we assume a sequence of training sessions with the corresponding datasets $\{\gD^{(0)},\gD^{(1)},\cdots,\gD^{(T)}\}$.
$\gD^{(t)}$ consists of training examples $\vx^{(t)}_i$ with its class labels $y_i^{(t)}\in\gC^{(t)}$ (for simplicity, we will exclude the superscript), where $\gC^{(t)}$ is the set of classes in its respective dataset $\gD^{(t)}$ and $\gC^{(s)}\cap\gC^{(t)}=\emptyset$ for $s \neq t$ (each dataset has its own distinct classes without overlap).
In the first session (a.k.a. base session) with the dataset $\gD^{(0)}$, there is assumed to be a large number of classes available with an abundant amount of training data for each class.
In subsequent sessions (a.k.a. incremental sessions) with the datasets $\gD^{(\ge1)}$, it is assumed that each dataset has a few training examples for each class.
In particular, FSCIL is said to have a $N$-way $K$-shot setting when each incremental session has $N$ classes with $K$ examples for each class.
At each $t$-th training session, only its corresponding dataset $\gD^{(t)}$ is accessible for training.
After each $t$-th session, the evaluation is performed on all previously seen classes $\gC^{(\le t)}$ using test datasets $\gD_{test}^{(\le t)}$, which consists of test examples with the class label set $\gC^{(\le t)}$.

\subsection{Background: Baseline}\label{sec:baseline}
Let a classification network consist of a feature extractor $f_{\boldsymbol{\theta}}(\cdot)$ and a classification layer with its weights $\boldsymbol{\phi}$.
The training objective of the base session is simply the softmax cross-entropy (SCE) loss with the cosine similarity $\texttt{sim}(\cdot,\cdot)$ as logits~\cite{deng2019arcface}:
\begin{equation}\label{eq:cross-entropy}
    \mathcal{L}_{\text{ce}} = \frac{1}{B} \sum_{i=1}^{B} -\log \frac{\exp(\frac{1}{\tau}\texttt{sim}(\vz_i,\boldsymbol{\phi}_{i}))}{\sum_{j = 1}^{\lvert \gC^{(0)} \rvert} \exp(\frac{1}{\tau}\texttt{sim}(\vz_i,\boldsymbol{\phi}_{j}))},
\end{equation}
where $\vz_i=f_{\boldsymbol{\theta}}(\vx_i)$; $B$ is the batch size and $\tau$ is the temperature parameter.
Incrementally updating weights with few examples in incremental sessions can make the network vulnerable to both catastrophic forgetting and overfitting.
To bypass the problems, several works~\cite{zhang2021cec,hersche2022constrained} suggest minimizing weight updates by freezing the feature extractor after the base session and using feature-average class prototype representation~\cite{snell2017prototypical}.
Specifically, after the base session, trained $\boldsymbol{\phi}$ is replaced with class prototypes, a process we refer to as classifier replacement (CR), and new-class prototypes are obtained in the subsequent incremental sessions.
The $i$-th class prototype is acquired by averaging the features of training samples of the $i$-th class:
\begin{equation}\label{eq:prototype}
    \boldsymbol{\phi}^{P}_{i} = \frac{1}{N_{c_i}} \sum\limits_{\substack{(\vx_j, y_j) \in \gD^{(\geq 0)}}} \mathbbm{1}_{[y_j = i]} f_{\boldsymbol{\theta}}(\vx_j),
\end{equation}
where $N_{c_i}$ is the number of training samples associated with the $i$-th class and $\mathbbm{1}_{[\cdot]}$ indicates 1 if the subscript condition is \texttt{True} and 0 otherwise.
For an input $\vx$, the classification score for the $i$-th class is computed by $\texttt{sim}(f_{\boldsymbol{\theta}}(\vx), \boldsymbol{\phi}^{P}_{i})$.

Although this baseline bypasses the forgetting and overfitting issues, it heavily relies on the quality of the representation trained solely on the base classes.
Consequently, the main focus of this paper is to investigate the important factors that influence representation learning for FSCIL and strategies to improve them.

\input{camera_ready/Figure/ssltemp}
\subsection{Transferability, feature spread, and its adverse effects on FSCIL}\label{sec:spread}

As shown in Fig.~\ref{fig:main}a, the baseline method exhibits a narrow intra-class distribution, widely perceived as representation collapse~\cite{papyan2020prevalance}.
Recent studies~\cite{chen2022perfectly,islam2021broad,xue2023which} have revealed that the collapsed representation shows poor transferability due to the loss of shareable low- and mid-level features that new classes can benefit from.
As such, they suggest joint optimization with a self-supervised contrastive (SSC) task~\cite{chen2020simple,he2020momentum}, which promotes the spread of features and thus the sharing of features among classes.
SSC learning optimizes infoNCE loss~\cite{oord2018representation}, regarding an augmented view from a query image as positive and the other images as negative samples.
The SSC loss for a positive pair $(i,j)$ is:
\begin{equation}
    \mathcal{L}_{\text{ssc}}^{(i,j)} = -\log \frac{\exp(\frac{1}{\tau}\texttt{sim}(\vz_i,\vz_j))}{\sum_{k=1}^{B}  \mathbbm{1}_{[k \neq i]} \exp(\frac{1}{\tau}\texttt{sim}(\vz_i,\vz_k))},
\end{equation}
where $B$ is the number of samples, including augmented images, and $\vz_j$ is the feature from an augmented view of $\vx_i$.
The loss is averaged over all positive pairs. 

In parallel, several studies have proposed to reduce the temperature and margin parameters in the softmax cross-entropy loss as another way to encourage feature sharing~\cite{liu2020negative,zou2022margin,kornblith2021why}.
Our empirical analysis in Section~\ref{sec:supp:margintemp} demonstrates that the temperature parameter affects the transferability more than the margin.
Thus, we employ a low-temperature parameter in the SCE loss to further enhance the transferability of learned representations.
Indeed, we observe that the combination of $\mathcal{L}_{\text{ssc}}$ and $\mathcal{L}_{\text{ce}}$ with low $\tau$ results in better new-class performance in the FSCIL problem due to a larger spread of features, as demonstrated in the left figure in Fig.~\ref{fig:ssctemp} and Fig.~\ref{fig:main}b.

However, as illustrated in the center figure in Fig.~\ref{fig:ssctemp}, we observe that the pursuit of better transferability results in a degradation in discriminability on the base classes.
While exploring the dilemma, we discover that the performance decline on the base classes mainly arises from the base class classifier replacement (CR) strategy, which is introduced in Section~\ref{sec:baseline}.
The figure on the right in Fig.~\ref{fig:ssctemp} shows the accuracy on the base classes subsequent to the base session training ($A_B^{0}$), without considering the new classes, both before and after CR, denoted by $A_B^{0,\text{bef}}$ and $A_B^{0,\text{aft}}$, respectively.
Before CR, $A_B^{0,\text{bef}}$ exhibits a relatively high value and a small variance as $\tau$ and $\lambda_{\text{ssc}}$ vary (indicated by the solid lines).
Conversely, $A_B^{0,\text{aft}}$ is considerably lower than $A_B^{0,\text{bef}}$ with a relatively higher variance (indicated by the dotted lines).
Considering that CR is introduced to bypass overfitting and catastrophic forgetting in the FSCIL problem, these findings suggest that the previous methods for improving the transferability of representations are not effective in enhancing the trade-off between transferability and discriminability in the context of the FSCIL problem.
Thus, it is essential to develop a representation learning method tailored for the FSCIL problem.

\input{camera_ready/Figure/inter_intra}
\subsection{Inter-Class Distance Matters}\label{sec:inter}
\textcolor{black}{As discussed in Section~\ref{sec:spread}, learning shareable features through representation spreading proves advantageous for transferability.
Yet, it also harms discriminability in the process of CR.}
Based on the observation in Fig.~\ref{fig:main}b, we hypothesize that such a dilemma appears to arise from a large inter-class distance since the representation spreading could push features into the extensive inter-class space, which may dilute the information on the base classes.
Furthermore, we argue that the large inter-class distance may impede effective feature sharing among classes, undermining the transferability of learned representations.
Consequently, to retain the knowledge on base classes while promoting effective feature sharing, we introduce a novel loss function that minimizes the inter-class distance:
\begin{equation}
    \mathcal{L}_{\text{inter}} = - \frac{1}{\sum\limits_{i=1}^B \sum\limits_{j>i}^B \mathbbm{1}_{[y_i \neq y_j]}} \sum\limits_{i=1}^B \sum\limits_{j>i}^B \mathbbm{1}_{[y_i \neq y_j]} \texttt{sim}(\vz_i,\vz_j).
\end{equation}
Fig.~\ref{fig:main}c displays that the spread of intra-class features is well regulated by applying $\mathcal{L}_{\text{inter}}$ and Fig.~\ref{fig:intra_inter} demonstrates the performance decline from CR is alleviated by minimizing $\mathcal{L}_{\text{inter}}$ (indicated by the \textcolor{violet}{purple} line).
Moreover, the results indicated by the \textcolor{SkyBlue}{skyblue} line show that reducing inter-class distance improves the performance on the new classes, corroborating our hypothesis.

\input{camera_ready/Figure/analysis}
Our assertion initially seems counter-intuitive, diverging from the common belief of prior works~\cite{yang2023neural,hersche2022constrained,zhou2022forward,song2023learning} that maximizing inter-class distance may be beneficial for reserving representation space for future new classes.
To further validate the efficacy of reducing inter-class distance with respect to transferability, we propose a measure to quantify how new class samples are distinct from base-class representations.
We define the measure as the averaged relative angular distance between a new-class sample and its nearest base-class prototype with respect to the averaged angular distance among all base-class prototype pairs:
\begin{equation}\label{eq:metric}
    \gT(f_{\boldsymbol{\theta}}) = \frac{\frac{1}{\lvert \gD_{test}^{(> 0)} \rvert} \sum_{(\vx_j, y_j) \in \gD_{test}^{(> 0)}} \min\limits_{i} \angle(\vz_j, \boldsymbol{\phi}^P_{base,i})}{\sum_{j=1}^{\lvert \gC^{(0)} \rvert} \sum_{k>j}^{\lvert \gC^{(0)} \rvert} \angle(\boldsymbol{\phi}^P_j,\boldsymbol{\phi}^P_k) \: / \: \binom{\lvert \gC^{(0)} \rvert}{2}},
\end{equation}
where $\boldsymbol{\phi}^P_{base,i}$ and $\angle(\cdot,\cdot)$ indicate the $i$-th base-class prototype and the angular distance between two input vectors, respectively.
We introduce the denominator to normalize the varying sizes of the representation space depending on methods.
If the representation produced by $f_{\boldsymbol{\theta}}$ has a distinguishable representation for new classes, $\gT(f_{\boldsymbol{\theta}})$ would be large.
For the sanity test, we check the relationship between $\gT(f_{\boldsymbol{\theta}})$ and the accuracy of the new class, which is shown in Fig.~\ref{fig:sanity_test}.

Using this measure, we analyze the relationship between inter-class distance, the spread of features, and the transferability of learned representation.
To do so, we train the feature extractor using $\mathcal{L}_{\text{ce}}$ with a low temperature, $\mathcal{L}_{\text{ssc}}$, and $\mathcal{L}_{\text{inter}}$ with varying loss weights for $\mathcal{L}_{\text{ssc}}$ and $\mathcal{L}_{\text{inter}}$, denoted by $\lambda_{\text{ssc}}$ and $\lambda_{\text{inter}}$, respectively.
Fig.~\ref{fig:inter_T} and Fig.~\ref{fig:T_new} show the relationship between the inter-class distance, $\gT(f_{\boldsymbol{\theta}})$, and the performance on the new classes.
The results demonstrate that the joint optimization of $\mathcal{L}_{\text{ssc}}$ and $\mathcal{L}_{\text{inter}}$ leads to an increase in both $\gT(f_{\boldsymbol{\theta}})$ and $A_N$, indicating that smaller inter-class distances lead to more discriminability between base classes and new classes.
The analysis corroborates our seemingly counter-intuitive claim that the closer classes are, the better.

In summary, we have observed that the spread of representation achieved by employing a low temperature in the SCE loss and self-supervised contrastive loss is advantageous for learning transferable representation.
However, spread representation itself cannot address the trade-off between transferability to new classes and discriminability on base classes in the context of the FSCIL problem.
Our analysis demonstrates that decreasing inter-class distance enhances discriminability by regularizing the intra-class spread and improves the transferability by promoting the effective learning of shareable information among classes when combined with the feature-spread-encouraging loss.
Consequently, our final loss function is the combination of cross-entropy loss with a lower temperature parameter, self-supervised contrastive loss, and inter-class distance minimizing loss:
\begin{equation}\label{eq:final}
    \mathcal{L} = \mathcal{L}_{\text{ce}} + \lambda_{\text{ssc}} \mathcal{L}_{\text{ssc}} + \lambda_{\text{inter}}\mathcal{L}_{\text{inter}}.
\end{equation}

\subsection{Information Bottleneck Theory Perspective}\label{sec:ib}
In this subsection, we provide theoretical support for the proposed representation learning objective in Eq.~\eqref{eq:final} from the perspective of the information bottleneck (IB) theory~\cite{tishby2015deep,alemi2017deep,shwartzziv2017opening}.
The IB theory describes the goal of representation learning as finding a good trade-off between complexity and accuracy through finding minimal information from inputs necessary to preserve maximal information about the targets.
As discussed by Cui~\textit{et al.}~\cite{cui22discriminability}, we consider the objective of IB theory to be closely related to learning transferable representations since while achieving the objective, a network could be guided to learn intrinsic knowledge rather than task-irrelevant shortcuts~\cite{geirhos2020shortcut}, leading to better transferability.

For an image classification task, let $X \in\mathbb{R}^{h\times w\times 3}$, $Y \in\mathbb{R}^C$, and $Z=\frac{f_{\boldsymbol{\theta}}(X)}{\lVert f_{\boldsymbol{\theta}}(X) \rVert} \in \mathbb{R}^d$ denote the input, label, and normalized latent representation variables, respectively, where $h$ and $w$ are the spatial sizes of the image, $C$ is the number of classes, and $d$ is the dimension of the latent representations.
The aforementioned trade-off has been formulated as the following objective:
\begin{equation}\label{eq:ib_inequality}
    \max I(Y;Z)-\beta I(X;Z),
\end{equation}
where $I(\cdot;\cdot)$ indicates the mutual information between two random variables and $\beta>0$ is a Lagrange multiplier.
In this work, we consider an alternative trade-off objective, $\max \frac{I(Y;Z)}{\beta I(X;Z)}$, which is adopted by several works~\cite{slonim1999agglomerative,Ngampruetikorn2022information}.
After omitting $\beta$ for simplicity and modest assumptions, we derive the following lower bound of the IB trade-off objective:
\begin{equation}\label{eq:ib_inequality}
    \frac{I(Y;Z)}{I(X;Z)} \geq 1 - \frac{d\cdot\log(2\pi e)+\frac{1}{C} \sum\limits_{i=1}^{C} \log\lvert \Sigma_{W_i} \rvert}{d\cdot\log(2\pi e) + \log\lvert \Sigma_{T} \rvert},
\end{equation}
where $\Sigma_{W_i}$ and $\Sigma_T$ indicate the covariance matrices of $Z$ within the $i$-th class and overall classes, respectively.
As proven by Lemma~\ref{sec:supple:proof}, both the numerator and denominator in the fractional term of the lower bound in Eq.~\eqref{eq:ib_inequality} are negative, leading to the following theorem:
\begin{theorem}\label{theorem:main}
The lower bound of $\frac{I(Y;Z)}{I(X;Z)}$ in \text{Eq.~\eqref{eq:ib_inequality}} is a monotonically increasing function of $\lvert \Sigma_{W_i} \rvert$ and a monotonically decreasing function of $\lvert \Sigma_{T} \rvert$.
\end{theorem}
The detailed derivations are presented in Section~\ref{sec:supple:proof}.

\input{camera_ready/Table/cub200}

Generally, the determinant of a covariance matrix is correlated with the spread or variability of variables across all dimensions.
Thus, Theorem~\ref{theorem:main} suggests that the IB trade-off can be enhanced by increasing the intra-class variability ($\mathcal{L}_{\text{ce}}$ with a low temperature and $\mathcal{L}_{\text{ssc}}$) while suppressing the overall representation space ($\mathcal{L}_{\text{inter}}$), supporting the proposed learning objective in Eq.~\eqref{eq:final} and our claims for inter-class distance minimization.
Intuitively, our objective functions encourage greater overlap in representations among different classes, promoting more shareable and compact representations.
Our theoretic analysis also underpins the findings in Fig.~\ref{fig:T_analysis} that reducing inter-class distance without representation spreading loss is observed to result in less transferable representations.
Since compressing the overall feature space diminishes the intra-class variability, the IB trade-off may decrease, leading to degraded transferability.

%% file: camera_ready/Figure/ssltemp.tex
\begin{figure*}[t!]
    \centering
    \includegraphics[width=0.9\linewidth]{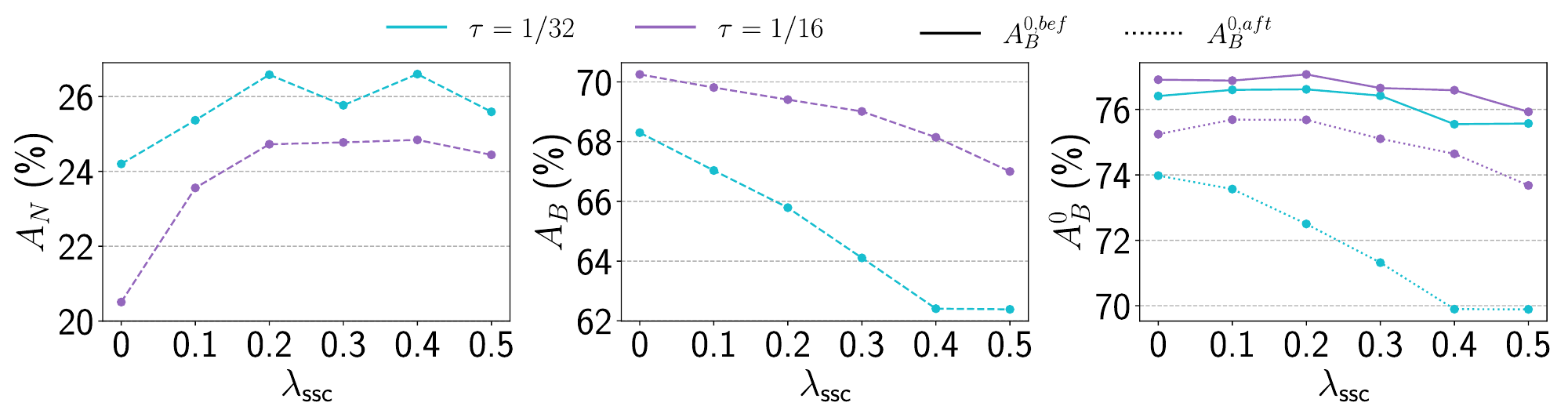}
    \caption{\textbf{The impact of the spread of representation.}
    Stronger emphasis on self-supervised contrastive loss (larger $\lambda_\text{ssc}$) and low temperature (\textcolor{SkyBlue}{skyblue}) enhances the new-class performance $A_N$ (\textbf{left}), but at the expense of base-class performance $A_B$ (\textbf{center}).
    The reduced base-class performance is mainly attributed to the excessive intra-class variation, adversely affecting the class prototype representation (\textbf{right}).
    The experiments are conducted on CIFAR100 dataset.
    }
    \label{fig:ssctemp}
\end{figure*}

%% file: camera_ready/Figure/inter_intra.tex
\begin{figure}[t!]
    \centering
    \includegraphics[width=0.5\textwidth]{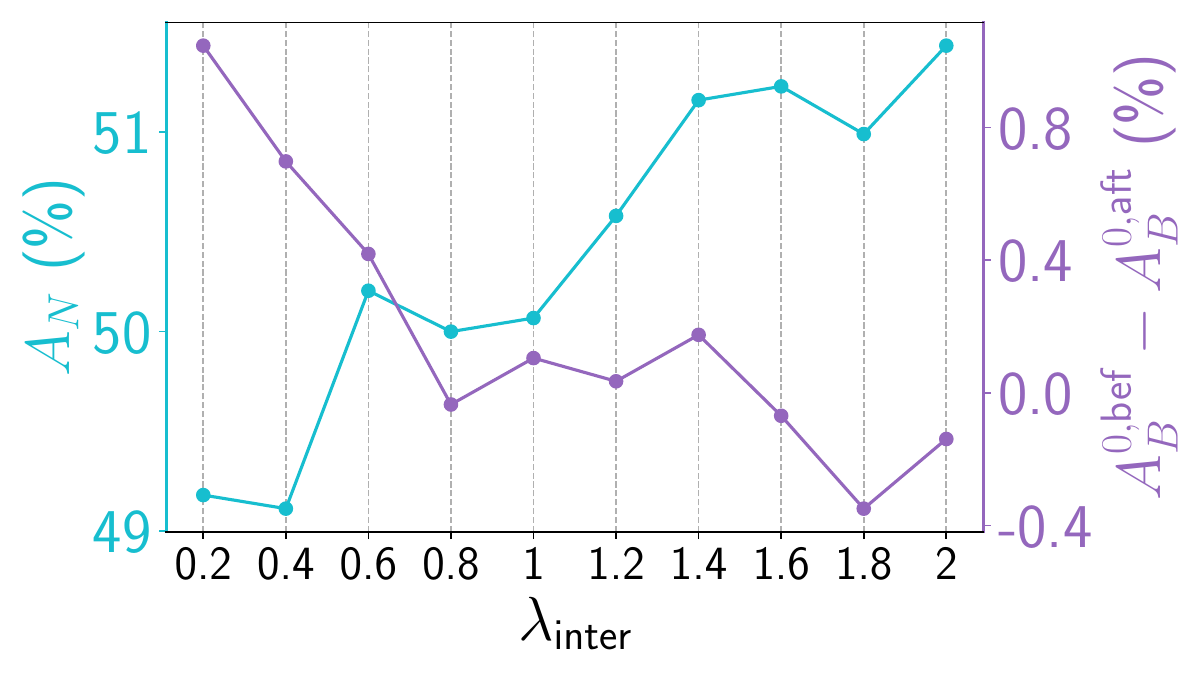}
    \caption{\textbf{Effect of minimizing inter-class distance.}
    As the weight of $\mathcal{L}_{\text{inter}}$, denoted by $\lambda_{\text{inter}}$, increases, the performance on the new classes increases (\textcolor{SkyBlue}{skyblue}) and the performance loss on the base classes induced by CR is greatly alleviated (\textcolor{violet}{purple}).
    The experiments are conducted on CUB200 dataset.
}
\label{fig:intra_inter}
\end{figure}

%% file: camera_ready/Figure/analysis.tex
\begin{figure*}[t!]
    \centering
    \begin{subfigure}[t!]{0.32\linewidth}
    	\includegraphics[width=\linewidth]{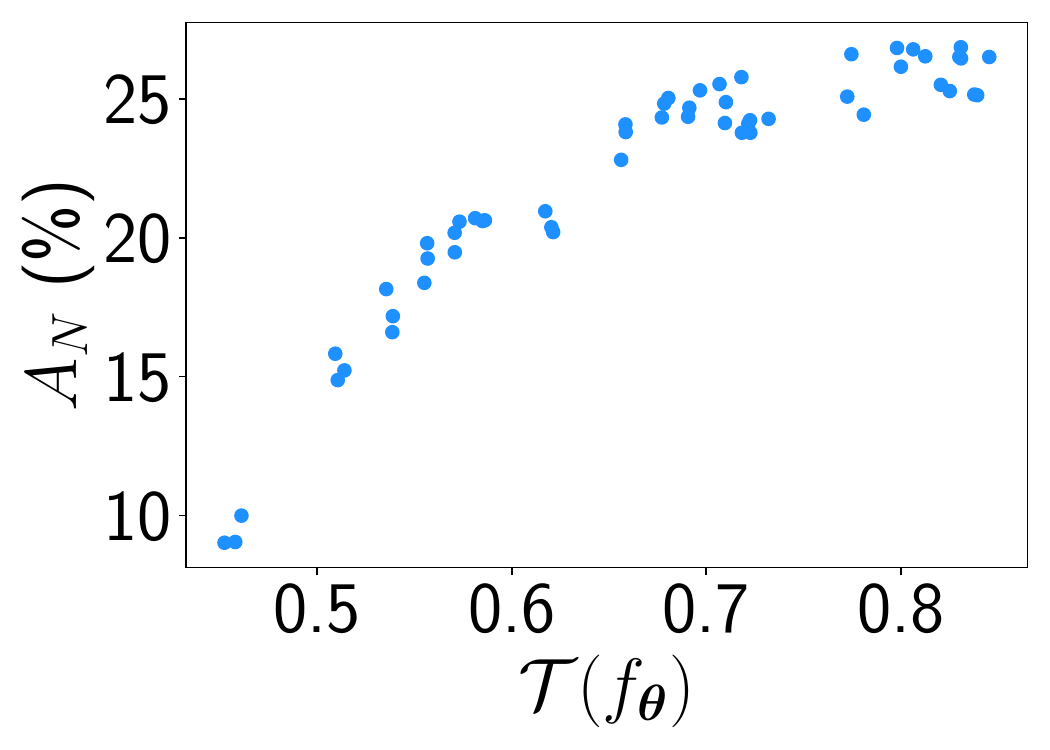}
	   \caption{}
    	\label{fig:sanity_test}
    \end{subfigure}
    \begin{subfigure}[t!]{0.32\linewidth}
	   \includegraphics[width=\linewidth]{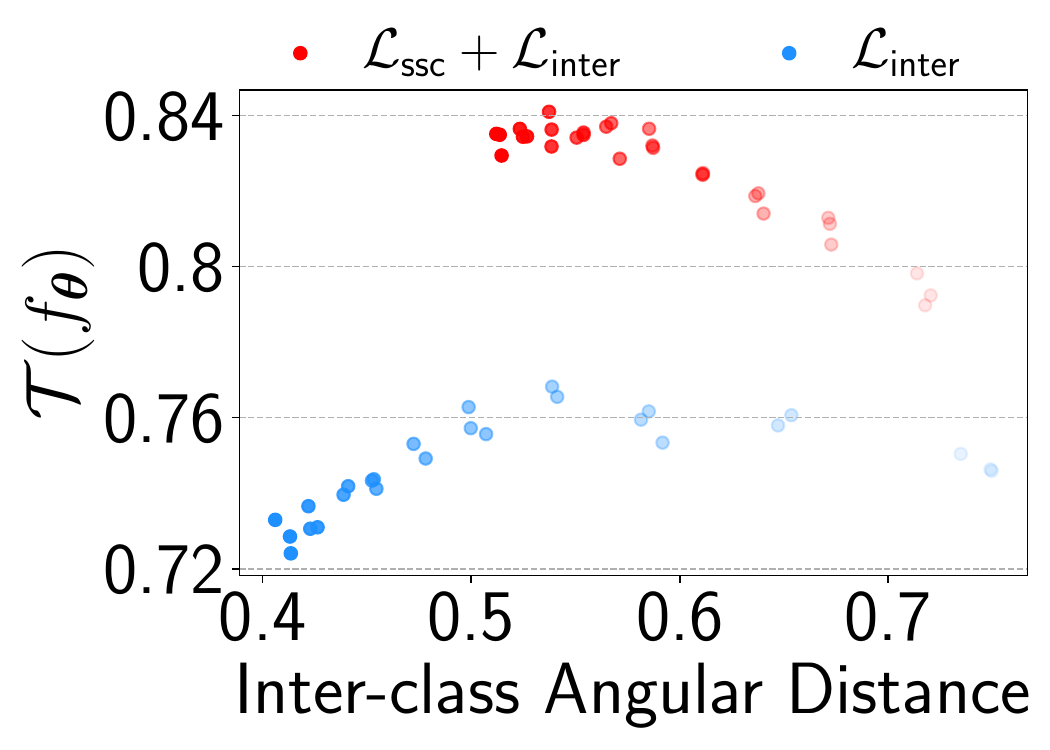}
	   \caption{}
	   \label{fig:inter_T}
    \end{subfigure}
    \begin{subfigure}[t!]{0.32\linewidth}
	   \includegraphics[width=\linewidth]{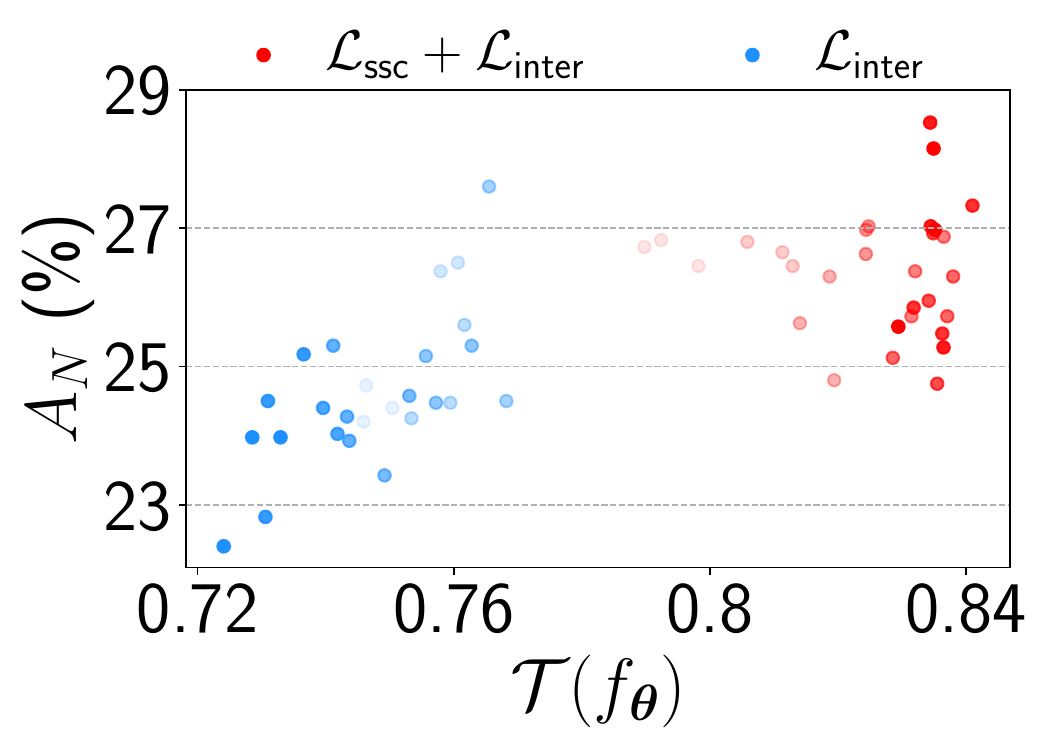}
	   \caption{}
	   \label{fig:T_new}
    \end{subfigure}
    \caption{\textbf{(a) Sanity test for $\gT(f_{\boldsymbol{\theta}})$}:
    $\gT(f_{\boldsymbol{\theta}})$ has a positive correlation with the performance on the new classes.
    Each data point is obtained by different configurations of $\tau$ and $\lambda_{\text{ssc}}$ (without $\mathcal{L}_{\text{inter}}$).
    \textbf{(b),(c) Relationship between inter-class distance, $\gT(f_{\boldsymbol{\theta}})$, and $A_N$}: Integrated with the representation spreading, reducing inter-class distance encourages better transferability (\textcolor{red}{red} points).
    However, the tendency is broken when reducing inter-class distance without representation spreading (\textcolor{NavyBlue}{blue} points).
    Please refer to Section~\ref{sec:ib} for theoretical support for these observations.
    The dots with greater transparency correspond to smaller $\lambda_{\text{inter}}$, ranging from 0 to 1 with intervals of 0.1.
    We set $\lambda_{\text{ssc}}$ as 0.1 when it is used.
    The experiments are conducted on CIFAR100 dataset.
    }
    \label{fig:T_analysis}
\end{figure*}

%% file: camera_ready/Table/cub200.tex
\begin{table*}[t!]
   \centering
   \scriptsize
   \caption{\textbf{10-way 5-shot incremental learning results on CUB200.}}
   \resizebox{0.9\linewidth}{!}{
   \begin{tabular}{lcccccccccccc}
   \toprule[\heavyrulewidth]
        \multicolumn{1}{c}{\multirow{2}{*}{Method}} & \multicolumn{11}{c}{Acc. in each session (\%) } &
        \multicolumn{1}{c}{\multirow{2}{*}{PD (\%) $\downarrow$ }}\\
        \cmidrule{2-12}
        \multicolumn{1}{c}{} & 0 & 1 & 2 & 3 & 4 & 5 & 6 & 7 & 8 & 9 & 10\\
   \toprule
   Baseline & 79.92 & 76.23 & 73.18 & 69.45 & 67.83 & 65.74 & 64.54 & 63.33 & 61.56 & 61.27 & 60.10 & 19.83\\
   \toprule
   TOPIC~\cite{tao2020topic} & 68.68 & 62.49 & 54.81 & 49.99 & 45.25 & 41.40 & 38.35 & 35.36 & 32.22 & 28.31 & 26.26 & 42.42\\
   F2M~\cite{shi2021overcoming} & 81.07 & 78.16 & 75.57 & 72.89 & 70.86 & 68.17 & 67.01 & 65.26 & 63.36 & 61.76 & 60.26 & 20.81\\
   CEC~\cite{zhang2021cec} & 75.85 & 71.94 & 68.50 & 63.50 & 62.43 & 58.27 & 57.73 & 55.81 & 54.83 & 53.52 & 52.28 & 23.57\\
   IDLVQ-C~\cite{chen2021incremental} & 77.37 & 74.72 & 70.28 & 67.13 & 65.34 & 63.52 & 62.10 & 61.54 & 59.04 & 58.68 & 57.81 & 19.56\\
   ALICE~\cite{peng2022fewshot} & 77.40 & 72.70 & 70.60 & 67.20 & 65.90 & 63.40 & 62.90 & 61.90 & 60.50 & 60.60 & 60.10 & \underline{17.30}\\
   CLOM~\cite{zou2022margin}& 79.57 & 76.07 & 72.94 & 69.82 & 67.80 & 65.56 & 63.94 & 62.59 & 60.62 & 60.34 & 59.58 & 19.99 \\
   Entropy Reg.~\cite{liu2022fewshot}& 75.90 & 72.14 & 68.64 & 63.76 & 62.58 & 59.11 & 57.82 & 55.89 & 54.92 & 53.58 & 52.39 & 23.51 \\
   LIMIT~\cite{zhou2022fewshot} & 75.89 & 73.55 & 71.99 & 68.14 & 67.42 & 63.61 & 62.40 & 61.35 & 59.91 & 58.66 & 57.41 & 18.48\\
   MetaFSCIL~\cite{chi2022metafscil} & 75.90 & 72.41 & 68.78 & 64.78 & 62.96 & 59.99 & 58.30 & 56.85 & 54.78 & 53.82 & 52.64 & 23.26\\
   FACT~\cite{zhou2022forward} & 75.90 & 73.23 & 70.84 & 66.13 & 65.56 & 62.15 & 61.74 & 59.83 & 58.41 & 57.89 & 56.94 & 18.96\\
   S3C~\cite{kalla2022s3c} & 80.62 & 77.55 & 73.19 & 68.54 & 68.05 & 64.33 & 63.58 & 62.07 & 60.61 & 59.79 & 58.95 & 20.83\\
   SAVC~\cite{song2023learning} & 81.85 & 77.92 & 74.95 & 70.21 & 69.96 & 67.02 & 66.16 & 65.30 & 63.84 & 63.15 & \underline{62.50} & 19.35\\
   NC-FSCIL~\cite{yang2023neural} & 80.45 & 75.98 & 72.30 & 70.28 & 68.17 & 65.16 & 64.43 & 63.25 & 60.66 & 60.01 & 59.44 & 21.01\\ 
   GKEAL~\cite{zhuang2023gkeal}& 78.88 & 75.62 & 72.32 & 68.62 & 67.23 & 64.26 & 62.98 & 61.89 & 60.20 & 59.21 & 58.67 & 20.21\\
   CABD~\cite{zhao2023fewshot}& 79.12 & 75.37 & 72.80 & 69.05 & 67.53 & 65.12 & 64.00 & 63.51 & 61.87 & 61.47 & 60.93 & 18.19\\
   \toprule
   \textbf{CLOSER (Ours)}& 79.40 & 75.92 & 73.50 & 70.47 & 69.24 & 67.22 & 66.73 & 65.69 & 64.00 & 64.02 & \textbf{63.58} & \textbf{15.82}\\
   \bottomrule[\heavyrulewidth]
   \end{tabular}
   }
   \label{tab:cub200}
\end{table*}

%% file: camera_ready/experiments.tex
\input{camera_ready/Table/cifar100}
\input{camera_ready/Table/mini_imagenet}

\section{Experiments}
\subsection{Experimental Details}\label{subsec:exp_details}

\noindent{\textbf{Dataset.}}
Following the benchmark settings proposed by Tao~\textit{et al.}~\cite{tao2020topic}, we evaluate the proposed method on CIFAR100~\cite{krizhevsky2009learning}, \textit{mini}ImageNet~\cite{vinyals2016matching}, and CUB200~\cite{WahCUB_200_2011}.
For CIFAR100 and \textit{mini}ImageNet, the total number of classes is 100: 60 base classes and 40 new classes.
The 40 new classes are split into 8 disjoint sets of 5 classes, each of which is sequentially provided with 5 training examples per class (5-way 5-shot) in each incremental session.
As for CUB200, there total number of classes is 200, with 100 base classes and 100 new classes.
100 new classes are split into 10 disjoint sets of 10 classes, each of which is sequentially provided with 5 training examples per class (10-way 5-shot) in each incremental session.

\noindent{\textbf{Implementation.}}
Following Zhang~\textit{et al.}~\cite{zhang2021cec}, we use ResNet-20~\cite{he2016resnet} for CIFAR100 experiments and ResNet-18~\cite{he2016resnet} for both \textit{mini}ImageNet and CUB200 experiments.
We follow the conventions to use the ResNet-18 model pre-trained on the ImageNet dataset~\cite{ILSVRC15} for CUB200.
\textcolor{black}{We set the mini-batch size to 128, 128, and 256 for CIFAR100, \textit{mini}ImageNet, and CUB200 experiments, respectively.}
The temperature parameter $\tau$ for the baseline method is $1/16$ and `low temperature' indicates $\tau=1/32$.
We set $\lambda_{\text{ssc}}$ as $0.1$,$0.1$, and $0.01$ for CIFAR100, \textit{mini}ImageNet, and CUB200, respectively, and $\lambda_{\text{inter}}$ as $1$, $0.5$, and $1.5$ for CIFAR100, \textit{mini}ImageNet, and CUB200, respectively.
These hyper-parameters are searched via validation using synthesized validation sets.
For mutual information estimation, we adopt MINE~\cite{belghazi18mutual} method and implement it based on the open-source code\footnote{\scriptsize \href{https://github.com/gtegner/mine-pytorch}{https://github.com/gtegner/mine-pytorch}}.
More details are provided in Section~\ref{sec:supple:details}.

\noindent{\textbf{Evaluation.}}
We use the accuracy on base ($A_B$), new ($A_N$), and the whole classes ($A_W$) as metrics to assess the discriminability, transferability, and the trade-off between them in learned representations.
Additionally, we use the performance drop (PD) between the accuracy at the end of the base session (session 0) and the accuracy at the last incremental session to evaluate the degree to which old knowledge is forgotten and new knowledge is learned simultaneously.
All experimental results of CLOSER are obtained by averaging 3 trials.

\input{camera_ready/Table/ablation}

\subsection{Comparison with the Existing Works}
We compare the proposed method, dubbed \textbf{CLOSER}, with prior arts on CUB200 (Table~\ref{tab:cub200}), CIFAR100 (Table~\ref{tab:cifar}), and \textit{mini}ImageNet (Table~\ref{tab:mini}).
We observe that CLOSER achieves state-of-the-art performance on both CUB200 and CIFAR100 datasets, surpassing the results of previous methods by a large margin with respect to $A_W$ and PD.
With \textit{mini}ImageNet, the proposed method exhibits substantially higher $A_W$ than the method with the lowest PD and achieves lower PD than the method with the highest $A_W$, which means that the proposed method achieves a better balance between the performance on base and the new classes.
It is worth noting that CLOSER shows such outstanding performance without any assistance from the storage of previous samples (F2M, ERDIL, IDLVQ-C, and CABD), additional computational modules (CEC, CLOM, NC-FSCIL, SAVC, and MetaFSCIL), and test-time data augmentation (S3C and SAVC), suggesting the critical importance of learning effective representations in FSCIL.

\subsection{Ablation Studies}
To verify the efficacy of the individual components in the proposed method, we perform ablation studies, which are shown in Table~\ref{tab:ablation}.
The increase in $A_N$ when employing a lower temperature in the softmax cross-entropy loss (low $\tau$) or a self-supervised contrastive loss ($\mathcal{L}_\text{ssc}$) confirms the advantage of feature spread in enhancing the transferability of representations.
While transferability can be greatly improved by using low $\tau$ and $\mathcal{L}_\text{ssc}$, a substantial decline in base-class performance $A_B$ is observed, as discussed in Section~\ref{sec:inter}.
The result from the case where all components are utilized demonstrates that this issue can be effectively resolved by minimizing inter-class distance.
Furthermore, as discussed in Section~\ref{sec:inter}, reducing inter-class distance is observed to improve transferability, especially when used with the representation spreading methods.
The comprehensive results of ablation studies confirm and support our claims.

\input{camera_ready/Figure/ib_analysis}
\subsection{Information Bottleneck trade-off Analysis}
In Section~\ref{sec:ib}, we show the connection between the proposed objective function and the information bottleneck (IB) trade-off.
To validate our analysis, we measure the mutual information between representations and inputs $I(X;Z)$ and the mutual information between representations and targets $I(Y;Z)$ for the baseline, representation spread approach, and our method CLOSER, as shown in Fig.~\ref{fig:ib}.
In the figure, the more left (lower $I(X;Z)$) and upper (higher $I(Y;Z)$) regions imply better IB trade-off.
Our proposed method CLOSER demonstrates to have found a better IB trade-off, especially when considering whole classes, including base and new classes.
The results are encouraging in that CLOSER is able to find a better IB trade-off for all classes, when the feature extractor is only trained on base classes and fixed afterwards.
The results demonstrate that CLOSER is effective in tackling a particularly difficult challenge of FSCIL: learning transferable and discriminative features from base classes.

\input{camera_ready/Figure/tsne}
\subsection{T-SNE Analysis}
For qualitative evaluation, we visualize the learned representation trained by different configurations of the proposed losses, which is illustrated in Fig.~\ref{fig:tsne}.
The value of $\gT(f_{\boldsymbol{\theta}})$ in the right below in each figure is for measuring the transferability of representation. 
The baseline representation (a) shows the characteristics of the base class, represented by the features of the new classes mapped to those of the base classes, also indicated by the low value of $\gT(f_{\boldsymbol{\theta}})$.
Spreading of features (b) largely resolves the overfitting issues, exhibited by larger distances between new classes and base classes; i.e., the increase in $\gT(f_{\boldsymbol{\theta}})$.
Finally, $\mathcal{L}_{\text{inter}}$ is used to compensate for the decreased performance on the base classes due to the spread representation.
Reducing inter-class distance also enhances the separability between the new class samples and the clusters of base classes, as evidenced by the further increase in $\gT(f_{\boldsymbol{\theta}})$.

%% file: camera_ready/Table/cifar100.tex
\begin{table*}[t!]
   \centering
   \scriptsize
   \caption{\textbf{5-way 5-shot incremental learning results on CIFAR100.}}
   \resizebox{0.8\linewidth}{!}{
   \begin{tabular}{lcccccccccc}
   \toprule[\heavyrulewidth]
        \multicolumn{1}{c}{\multirow{2}{*}{Method}} & \multicolumn{9}{c}{Acc. in each session (\%) } &
        \multicolumn{1}{c}{\multirow{2}{*}{PD (\%) $\downarrow$ }}\\
        \cmidrule{2-10}
        \multicolumn{1}{c}{} & 0 & 1 & 2 & 3 & 4 & 5 & 6 & 7 & 8 &\\
   \toprule
   Baseline & 72.93 & 68.46 & 64.26 & 60.15 & 56.53 & 53.60 & 51.51 & 49.19 & 47.09 & 25.84\\
   \toprule
   ERDIL~\cite{dong2021few} & 73.62 & 68.22 & 65.14 & 61.84 & 58.35 & 55.54 & 52.51 & 50.16 & 48.23 & 25.39\\
   CEC~\cite{zhang2021cec} & 73.07 & 68.88 & 65.26 & 61.19 & 58.09 & 55.57 & 53.22 & 51.34 & 49.14& 23.93\\
   CLOM~\cite{zou2022margin} & 74.20 & 69.83 & 66.17 & 62.39 & 59.26 & 56.48 & 54.36 & 52.16 & 50.25 & 23.95\\
   Entropy Reg.~\cite{liu2022fewshot} & 74.4 & 70.2 & 66.54 & 62.51 & 59.71 & 56.58 & 54.52 & 52.39 & 50.14 & 24.26 \\
   FACT~\cite{zhou2022forward} & 74.60 & 72.09 & 67.56 & 63.52 & 61.38 & 58.36 & 56.28 & 54.24 & \underline{52.10} & \underline{22.50}\\
   LIMIT~\cite{zhou2022fewshot} & 73.81 & 72.09 & 67.87 & 63.89 & 60.70 & 57.77 & 55.67 & 53.52 & 51.23 & 22.58\\
   MetaFSCIL~\cite{chi2022metafscil} & 74.50 & 70.10 & 66.84 & 62.77 & 59.48 & 56.52 & 54.36 & 52.56 & 49.97 & 24.53\\
   GKEAL~\cite{zhuang2023gkeal} & 74.01 & 70.45 & 67.01 & 63.08 & 60.01 & 57.30 & 55.50 & 53.39 & 51.40 & 22.61\\
   \toprule
   \textbf{CLOSER (Ours)}& 75.72 & 71.83 & 68.32 & 64.62 & 61.91 & 59.25 & 57.53 & 55.43 & \textbf{53.32} & \textbf{22.40}\\
   \bottomrule[\heavyrulewidth]
   \end{tabular}
   }
   \label{tab:cifar}
\end{table*}

%% file: camera_ready/Table/mini_imagenet.tex
\begin{table*}[t!]
   \centering
   \scriptsize
   \caption{\textbf{5-way 5-shot incremental learning results on \textit{mini}ImageNet.}}
   \resizebox{0.8\linewidth}{!}{
   \begin{tabular}{lcccccccccc}
   \toprule[\heavyrulewidth]
        \multicolumn{1}{c}{\multirow{2}{*}{Method}} & \multicolumn{9}{c}{Acc. in each session (\%) } &
        \multicolumn{1}{c}{\multirow{2}{*}{PD (\%) $\downarrow$ }}\\
        \cmidrule{2-10}
        \multicolumn{1}{c}{} & 0 & 1 & 2 & 3 & 4 & 5 & 6 & 7 & 8 &\\
   \toprule
   Baseline & 72.27 & 67.46 & 63.26 & 59.73 & 56.56 & 53.53 & 50.90 & 48.93 & 47.26 & 25.01\\
   \toprule
   TOPIC~\cite{tao2020topic} & 61.31 & 50.09 & 45.17 & 41.16 & 37.48 & 35.52 & 32.19 & 29.46 & 24.42 & 36.89\\
   F2M~\cite{shi2021overcoming} & 67.28 & 63.80 & 60.38 & 57.06 & 54.08 & 51.39 & 48.82 & 46.58 & 44.65 & 22.63\\
   CEC~\cite{zhang2021cec} & 72.00 & 66.83 & 62.97 & 59.43 & 56.70 & 53.73 & 51.19 & 49.24 & 47.63 & 24.37\\
   IDLVQ-C~\cite{chen2021incremental} & 64.77 & 59.87 & 55.93 & 52.62 & 49.88 & 47.55 & 44.83 & 43.14 & 41.84 & 22.93\\
   Subspace Reg.~\cite{akyurek2022subspace} & 80.37 & 73.76 & 68.36 & 64.07 & 60.36 & 56.27 & 53.10 & 50.45 & 47.55 & 32.83\\
   ALICE~\cite{peng2022fewshot} & 80.60 & 70.60 & 67.40 & 64.50 & 62.50 & 60.00 & 57.80 & 56.80 & \textbf{55.70} & 24.90 \\
   CLOM~\cite{zou2022margin}& 73.08 & 68.09 & 64.16 & 60.41 & 57.41 & 54.29 & 51.54 & 49.37 & 48.00 & 25.08 \\
   Entropy Reg.~\cite{liu2022fewshot}& 71.84 & 67.12 & 63.21 & 59.77 & 57.01 & 53.95 & 51.55 & 49.52 & 48.21 & 23.63 \\
   LIMIT~\cite{zhou2022fewshot} & 72.32 & 68.47 & 64.30 & 60.78 & 57.95 & 55.07 & 52.70 & 50.72 & 49.14 & 23.13\\
   MetaFSCIL~\cite{chi2022metafscil} & 72.04 & 67.94 & 63.77 & 60.29 & 57.58 & 55.16 & 52.90 & 50.79 & 49.19 & 22.85\\
   FACT~\cite{zhou2022forward} & 72.56 & 69.63 & 66.38 & 62.77 & 60.60 & 57.33 & 54.34 & 52.16 & 50.49 & \textbf{22.07}\\
   GKEAL~\cite{zhuang2023gkeal}& 73.59 & 68.90 & 65.33 & 62.29 & 59.39 & 56.70 & 54.20 & 52.59 & 51.31 & \underline{22.28}\\
   CABD~\cite{zhao2023fewshot}& 74.65 & 70.43 & 66.29 & 62.77 & 60.75 & 57.24 & 54.79 & 53.65 & 52.22 & 22.43\\
   \toprule
   \textbf{CLOSER (Ours)}& 76.02 & 71.61 & 67.99 & 64.69 & 61.70 & 58.94 & 56.23 & 54.52 & \underline{53.33} & 22.69\\
   \bottomrule[\heavyrulewidth]
   \end{tabular}
   }
   \label{tab:mini}
\end{table*}

%% file: camera_ready/Table/ablation.tex
\begin{table*}[t!]
\centering
\scriptsize
\caption{\textbf{Ablation studies on CIFAR100.}}
\resizebox{0.55\linewidth}{!}{
    \begin{tabular}{cccccccc}
      \toprule
      low $\tau$ & $\mathcal{L}_{\text{ssc}}$ & $\mathcal{L}_{\text{inter}}$ & $A_B$ (\%) & $A_N$ (\%) & $A_W$ (\%) & PD (\%)\\
      \midrule
      \color{black}\xmark & \color{black}\xmark & \color{black}\xmark & 70.13 & 20.95 & 50.46 & 24.81\\
      \color{black}\xmark & \color{green}\checkmark & \color{black}\xmark & 69.95 & 22.80 & 51.09 & 24.59\\
      \color{black}\xmark & \color{black}\xmark & {\color{green}\checkmark} & 69.95 & 20.05 & 49.99 & 26.78\\
      \color{black}\xmark & {\color{green}\checkmark} & {\color{green}\checkmark} & 71.43 & 22.83 & 51.99 & 25.26\\
      \color{green}\checkmark & \color{black}\xmark & \color{black}\xmark & 68.33 & 24.10 & 50.64 & 23.23\\
      \color{green}\checkmark & \color{green}\checkmark & \color{black}\xmark & 66.58 & 25.08 & 49.98 & 23.15\\
      {\color{green}\checkmark} & \color{black}\xmark & {\color{green}\checkmark} & 70.17 & 22.40 & 51.06 & 25.27\\
      {\color{green}\checkmark} & {\color{green}\checkmark} & {\color{green}\checkmark} & 70.72 & 27.23 & \textbf{53.32} & \textbf{22.40}\\
      \bottomrule
    \end{tabular}}\label{tab:ablation}
\end{table*}

%% file: camera_ready/Figure/ib_analysis.tex
\begin{figure*}[t!]
    \centering
    \begin{subfigure}[t!]{0.3\linewidth}
    	\includegraphics[width=\linewidth]{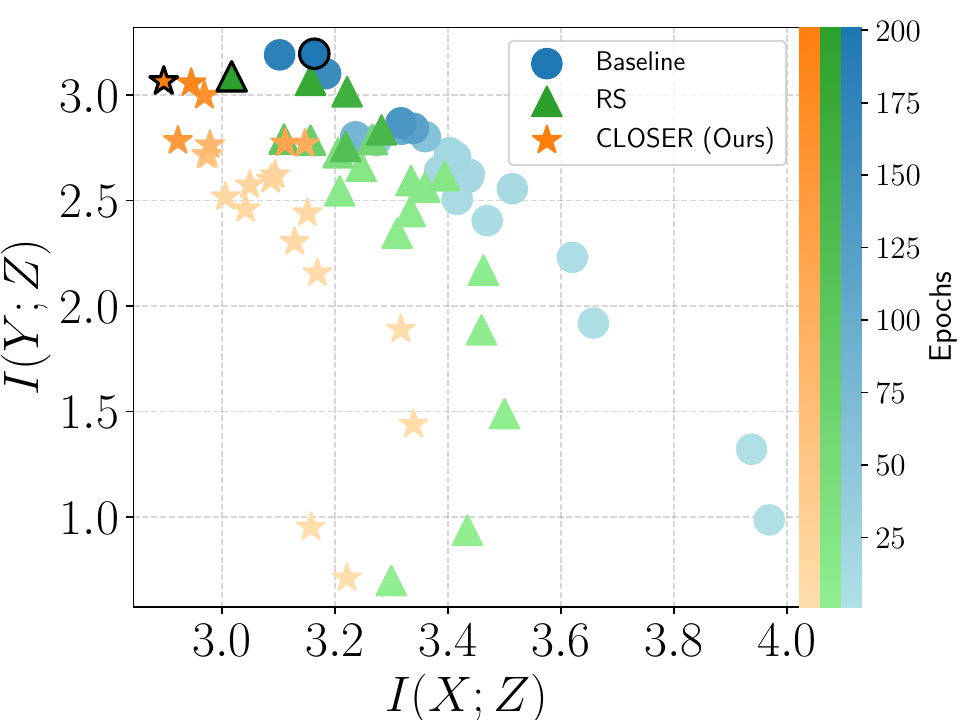}
	   \caption{Base Classes}
    	\label{fig:ib_base}
    \end{subfigure}
    \begin{subfigure}[t!]{0.3\linewidth}
	   \includegraphics[width=\linewidth]{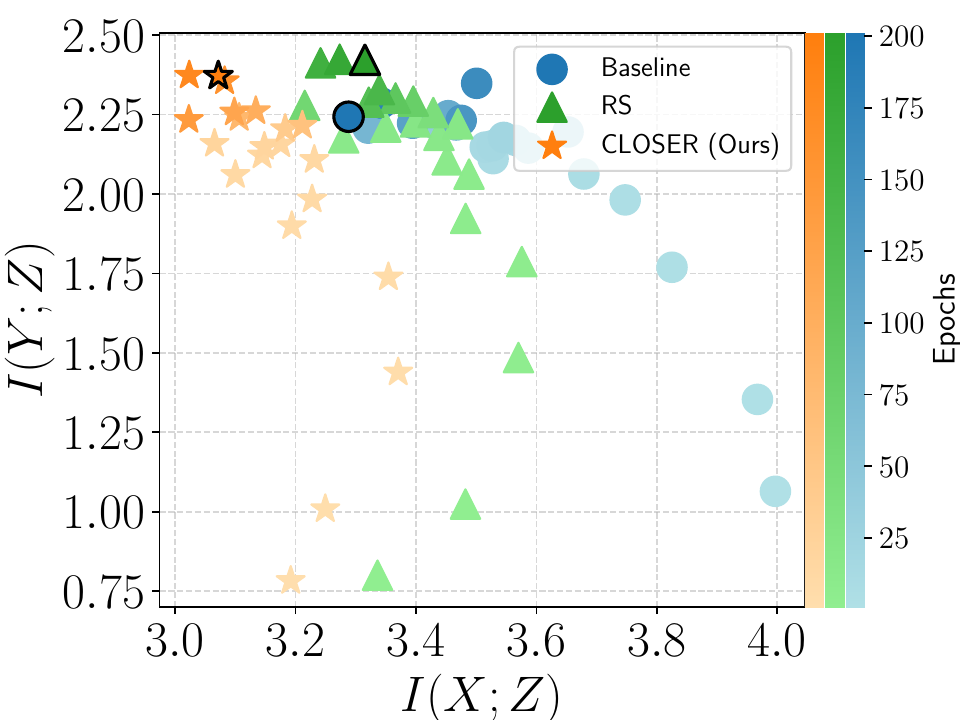}
	   \caption{New Classes}
	   \label{fig:ib_new}
    \end{subfigure}
    \begin{subfigure}[t!]{0.3\linewidth}
	   \includegraphics[width=\linewidth]{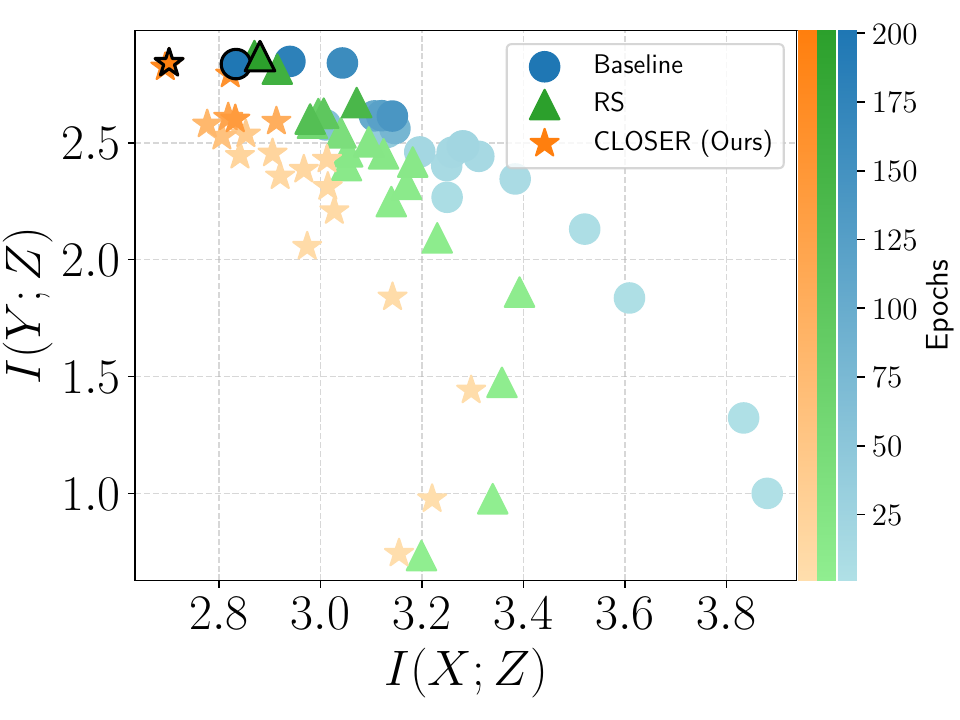}
	   \caption{Whole Classes}
	   \label{fig:ib_whole}
    \end{subfigure}
    \caption{\textbf{Information bottleneck trade-off analysis.}
    We compare representations acquired by three different methods by assessing the information bottleneck (IB) trade-off.
    `RS' refers to representation spreading methods.
    We indicate the final models with black edges.
    The experiments are conducted on the CIFAR100 dataset.
    }
    \label{fig:ib}
\end{figure*}

%% file: camera_ready/Figure/tsne.tex
\begin{figure*}[t!]
    \centering
        \begin{subfigure}{0.3\linewidth}
	   \includegraphics[width=\linewidth]{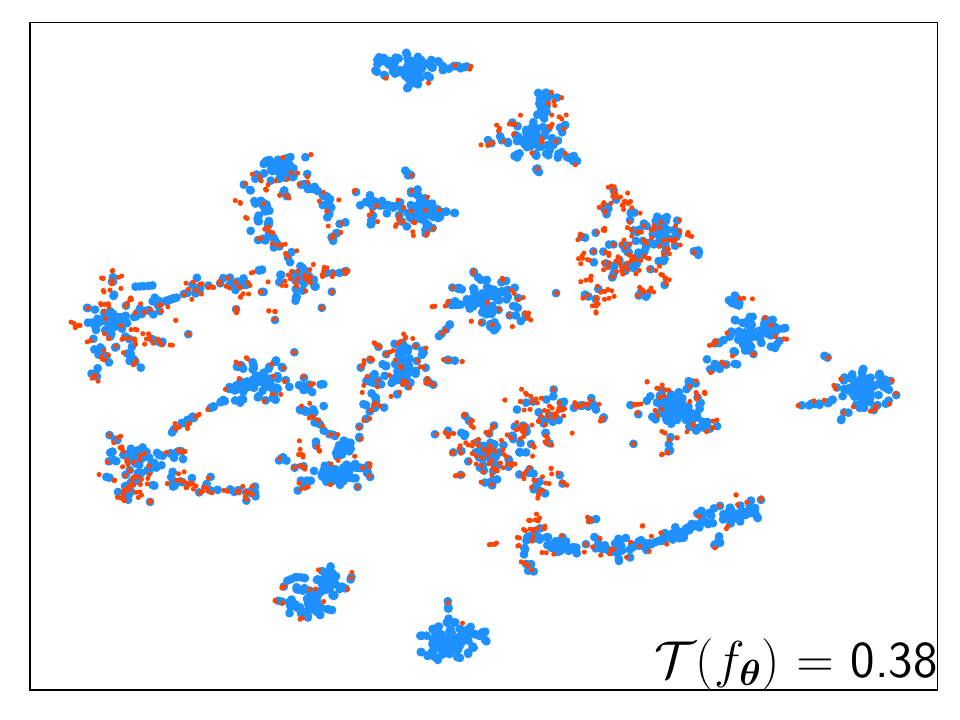}
	   \caption{Baseline}
        \label{fig:tsne_baseline}
    \end{subfigure}
    \begin{subfigure}{0.3\linewidth}
    	\includegraphics[width=\linewidth]{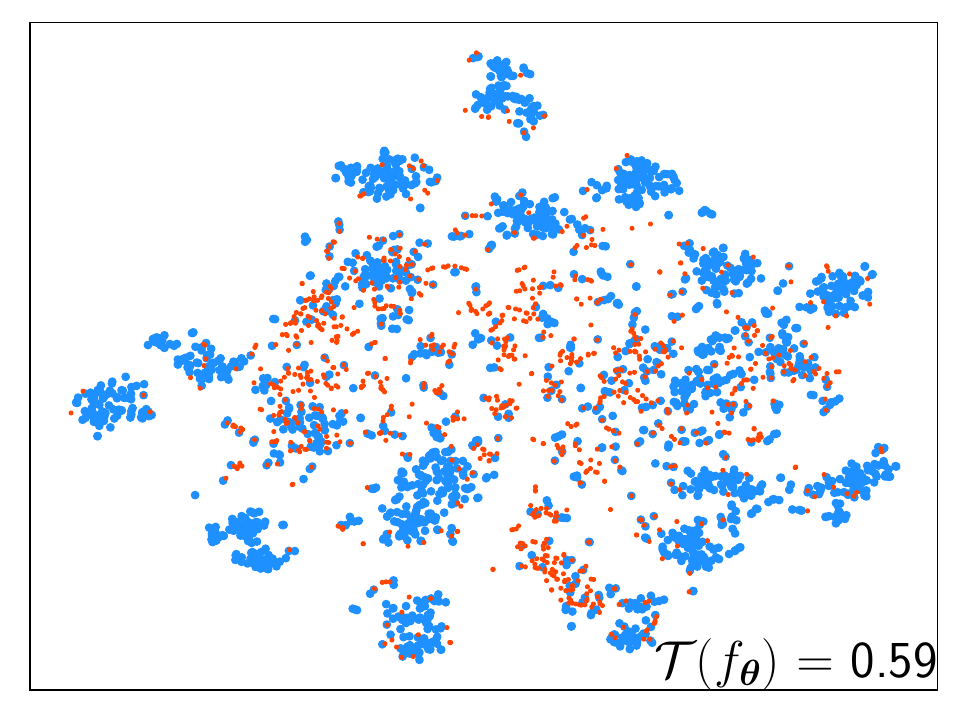}
        \caption{\textcolor{black}{Baseline + RS}}
        \label{fig:tsne_spread}
    \end{subfigure}
     \begin{subfigure}{0.3\linewidth}
	   \includegraphics[width=\linewidth]{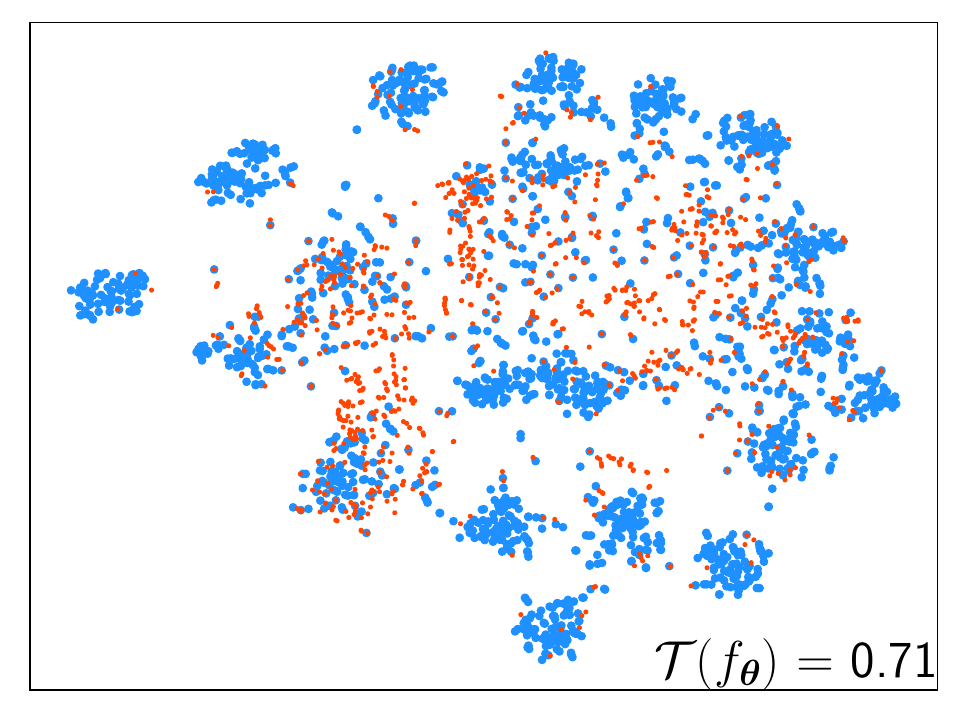}
        \caption{\textcolor{black}{CLOSER (Ours)}}
     \label{fig:tsne_ours}
    \end{subfigure}
    \caption{\textbf{T-SNE visualization of learned representations.}
    The blue and red points indicate the base and new class samples, respectively.
    We measure $\gT(f_{\boldsymbol{\theta}})$ to quantify the transferability of the learned representation.
    We conducted the experiments on CIFAR100 with reduced classes (20 base classes and 10 new classes) for better visualization.
    \textcolor{black}{The classification results for each experiment are as follows:
    \textbf{(a)}: $A_B$=78.85\%, $A_N$=13.50\%, $A_W$=57.07\%. 
    \textbf{(b)}: $A_B$=77.85\%, $A_N$=25.50\%, $A_W$=60.22\%. 
    \textbf{(c)}: $A_B$=78.88\%, $A_N$=28.80\%, $A_W$=62.18\%.}
}
	\label{fig:tsne}
\end{figure*}

%% file: camera_ready/conclusion.tex
\section{Limitations and Discussions}
Our work focuses on learning representation that is discriminative yet transferable to unseen classes to tackle the challenges of FSCIL.
As such, our work does not consider an option of updating representation with new classes.
Continual update of representation can improve the performance on new classes, however at the cost of sacrificing the old-class performance.
Thus, it is a question of stability-plasticity dilemma, where our method focuses more on stability while improving plasticity through discriminative and transferable representation.
Furthermore, similar to other works on FSCIL, our method is limited to classification tasks.
But, we believe our work can have an impact on other domains, akin to the impact of representation spread methods, such as contrastive learning. 
\section{Conclusion}
To tackle convoluted challenges in few-shot class-incremental learning (FSCIL), we focus on representation learning, which plays a crucial role.
In contrast to previous FSCIL methods that have focused on maximizing the distance between classes, our experimental and theoretic analysis suggests that the closer classes are, the better for FSCIL, especially when feature sharing is encouraged between classes.
Upon the analysis, we propose a simple, yet seemingly counter-intuitive idea: bring classes closer for a better transferability-discriminability Pareto front.
Our experimental results and information-bottleneck-theory-based analysis suggest that our work can provide a promising research avenue.
As such, we hope that our work will inspire future works and discussions in this research direction.

%% file: camera_ready_supple.tex
\setcounter{section}{0}
\setcounter{figure}{0}
\setcounter{table}{0}
\setcounter{equation}{0}

\renewcommand{\thetable}{S\arabic{table}}
\renewcommand{\thesection}{S\arabic{section}}
\renewcommand{\thefigure}{S\arabic{figure}}
\renewcommand{\theequation}{S\arabic{equation}}
\renewcommand{\thetheorem}{S\arabic{theorem}}
\renewcommand{\thelemma}{S\arabic{lemma}}

\newcommand\lreqn[2]{\noindent\makebox[\textwidth]{$\displaystyle#1$\hfill(#2)}\vspace{2ex}}

\input{supple/derivation}
\input{supple/generalization}
\input{supple/comparison}
\input{supple/margin_temp}
\input{supple/details}
\input{supple/Figure/margintemp_visualizations}

%% file: supple/derivation.tex
\section{Proof of Theorem~\ref{theorem:main}}\label{sec:supple:proof}
In this section, we present the proof of Theorem~\ref{theorem:main} of the main manuscript.
For an image classification task, let $X \in\mathbb{R}^{h\times w\times 3}$ and $Y \in\mathbb{R}^C$  denote the input and label, respectively.
Our goal is to train an encoder with parameters $\boldsymbol{\theta}$, $f_{\boldsymbol{\theta}}: \mathbb{R}^{h\times w\times 3} \rightarrow \mathbb{R}^d$.
The latent representation $Z$ is obtained by normalizing the network embedding, \ie $Z=\frac{f_{\boldsymbol{\theta}}(X)}{\lVert f_{\boldsymbol{\theta}}(X) \rVert} \in \mathcal{S}_d$ where $\mathcal{S}_d$ denotes the surface of the $d$-dimensional unit hypersphere.
Let $\Sigma_{W_i}$ and $\Sigma_{T}$ denote the covariance matrices of representations within the $i$-th class and whole classes, respectively.
We consider the trade-off objective of information-bottleneck (IB) theory as solving $\text{max}\: \frac{I(Y;Z)}{\beta I(X;Z)}$, where $I(\cdot;\cdot)$ denotes the mutual information between two variables and $\beta>0$, as discussed in the main manuscript.
After omitting $\beta$ for simplicity and modest assumptions, we prove the Theorem~\ref{theorem:main} as follows.

\vspace{1cm}
\begin{proof}
    Let $H(\cdot)$ denote a differential entropy of a continuous random variable.
    Then, $I(Y;Z)=H(Z)-H(Z \lvert Y)$ and $I(X;Z) = H(Z) - H(Z \lvert X)$.
    Since $f_{\boldsymbol{\theta}}$ is deterministic and there are finite examples in the dataset, $I(X;Z) = H(Z)$~\cite{cui22discriminability}.
    We then obtain:
    \begin{equation}\label{eq:s1}
    \begin{split}
        \frac{I(Y;Z)}{I(X;Z)} &= \frac{H(Z)-H(Z \lvert Y)}{H(Z)}\\
        & = 1 - \frac{H(Z \lvert Y)}{H(Z)}.
    \end{split}
    \end{equation}
    By representing $H(Z\lvert Y)$ as the weighted sum of the entropy of $Z$ conditioned on each possible value of $Y$, we derive
    \begin{equation}\label{eq:s2}
        \frac{I(Y;Z)}{I(X;Z)} = 1 - \frac{\sum\limits_{i=1}^C P(Y=y_i) H(Z\lvert Y=y_i)}{H(Z)},
    \end{equation}
    where $y_i \in \mathbb{R}^C$ is a one-hot vector containing a single 1 in the $i$-th element, hence denoting the label of $i$-th class.
    Using the property of differential entropy~\cite{elements2006cover}, we obtain the following inequality on $H(Z)$:
    \begin{equation}\label{eq:s3}
        H(Z) \leq \frac{d}{2} \log(2\pi e) + \frac{1}{2} \log \lvert \Sigma_T \rvert.
    \end{equation}
    By assuming that the representation distribution of each class follows a multivariate Gaussian distribution, the following equality holds for each $H(Z\lvert Y=y_i)$:
    \begin{equation}\label{eq:s4}
        H(Z\lvert Y=y_i) = \frac{d}{2} \log(2\pi e) + \frac{1}{2} \log \lvert \Sigma_{W_i} \rvert, \qquad i=\{1,2, \dotsm, C\}.
    \end{equation}
    \vspace{0.5cm}
    \begin{lemma}\label{lemma1}
        If $Z \in \mathcal{S}_d$ and $d \gg 1$, then $H(Z) < 0$.
    \end{lemma}
    \begin{proof}
        Let $f_Z$ denote the probability density function of a continuous variable $Z\in \mathcal{S}_d$.
        Then, we define the support of $Z$ as $\mathcal{D} = \{z\in \mathcal{S}_d\lvert f_Z(z)>0\}$.
        Since the maximum entropy within $\mathcal{D}$ is the entropy of a uniform distribution within $\mathcal{D}$, we derive
        \begin{equation}
            \begin{split}
                H(Z) & \leq H(\mathcal{U}_\mathcal{D})\\
                & = -\int_\mathcal{D} \frac{1}{V} \log \frac{1}{V} dD\\
                & = \log V,
            \end{split}    
        \end{equation}
        where $\mathcal{U}_\mathcal{D}$ denotes a uniform distribution defined in $\mathcal{D}$, $dD = dx_1dx_2\dotsm dx_d$, and $V = \int_\mathcal{D} dD$ is the volume of $\mathcal{D}$.
        Since $\mathcal{D}$ is a subset of $\mathcal{S}_d$, the maximum volume of $\mathcal{D}$ is the volume of $\mathcal{S}_d$, \ie $\text{max}\: V = \text{max}_\mathcal{D} \: \int_\mathcal{D} dD = \int_{\mathcal{S}_d} dD = \frac{2\pi^{d/2}}{\Gamma(d/2)}$ where $\Gamma(\cdot)$ is the Gamma function.
        Thus, $H(Z) \leq \log V \leq \log \frac{2\pi^{d/2}}{\Gamma(d/2)}$, leading to $H(Z) \leq \log \frac{2\pi^{d/2}}{\Gamma(d/2)} <0$  when $d$ is sufficiently large.
        \leavevmode\penalty50\hbox{}\nobreak\hfill $\square$
    \end{proof}
    Given that the equality of Eq.~\eqref{eq:s3} holds if $Z$ follows a multivariate Gaussian distribution within $\mathcal{S}_d$, Lemma~\ref{lemma1} proves that the upper bound of Eq.~\eqref{eq:s3} is also negative, leading to $H(Z) \leq \frac{d}{2} \log(2\pi e) + \frac{1}{2} \log \lvert \Sigma_T \rvert<0$.
    Thus, from Eq.~\eqref{eq:s3}, Eq.~\eqref{eq:s4}, and Lemma~\ref{lemma1} and assuming $P(Y=y_i)=\frac{1}{C}$ for all $i$, we can rewrite Eq.~\eqref{eq:s2} as follows:
    \begin{equation}
    \begin{split}
        \frac{I(Y;Z)}{I(X;Z)} & = 1 - \frac{\sum\limits_{i=1}^C P(Y=y_i) H(Z\lvert Y=y_i)}{H(Z)}\\
        & = 1 - \frac{\frac{d}{2} \cdot \log(2\pi e) + \frac{1}{C} \sum\limits_{i=1}^C \frac{1}{2} \log \lvert \Sigma_{W_i} \rvert}{H(Z)} \hspace{1em} \text{(Eq.~\eqref{eq:s4} and} \:P(Y=y_i) = \frac{1}{C})\\
        & \geq 1 - \frac{d\cdot \log(2\pi e) + \frac{1}{C}\sum\limits_{i=1}^C \log \lvert \Sigma_{W_i} \rvert}{d\cdot \log(2\pi e) + \log \lvert \Sigma_T \rvert} \hspace{4em} \text{(Eq.~\eqref{eq:s3} and Lemma~\ref{lemma1})}\\
    \end{split}
    \end{equation}
    Both the numerator and denominator in the fractional term of the lower bound are proven to be negative by Lemma~\ref{lemma1} and are monotonically increasing functions of $\lvert \Sigma_{W_i} \rvert$ and $\lvert \Sigma_T \rvert$, respectively.
    Therefore, the lower bound of $\frac{I(Y;Z)}{I(X;Z)}$ is a monotonically increasing function of $\lvert \Sigma_{W_i} \rvert$ and a monotonically decreasing function of $\lvert \Sigma_T \rvert$.
    \leavevmode\penalty50\hbox{}\nobreak\hfill $\square$
\end{proof}

%% file: supple/generalization.tex
\input{supple/Table/imagenet}

\input{supple/Table/vit}
\section{Generalization Ability of CLOSER}\label{sec:supple:generalization}
In this section, we examine the generalization-ability of CLOSER concerning both dataset and architecture.
To validate the effectiveness of CLOSER on a more challenging dataset, we construct ImageNet-FSCIL dataset, where the total 1000 classes of the ImageNet dataset are split into 600 base and 400 new classes.
The new classes are further divided into 10 disjoint sets, each set of which is sequentially provided with 5 training examples during the incremental sessions.
Using the ImageNet-FSCIL dataset, we compare CLOSER with baseline and baseline + representation spreading (RS) methods.
In detail, we train ResNet-18 using base-class samples from scratch during 90 epochs with SGD optimizer.
The learning rate is initially set to 0.1 and decays by a factor of 0.1 every 30 epochs.
Moreover, beyond the convolutional neural network, we explore the generalization capability of CLOSER on the popular Vision-Transformer (ViT) architecture~\cite{dosovitskiy2021vit}.
\Cref{tab:imagenet,tab:vit} show the results on ImageNet-FSCIL dataset and the ViT network, respectively.
We observe that representations acquired by the baseline method exhibit great discriminability on base classes, indicated by the relatively high $A_B$, but worst transferability to new classes and significant interference between the base and new classes, indicated by the lowest $A_N$ and the highest PD, respectively.
While the RS method is observed to enhance $A_N$ and PD, it compromises the discriminability on base classes, indicated by the lowest $A_B$.
On the other hand, our CLOSER achieves a significantly improved balance between discriminability and transferability, while also notably reducing the interference between base and new classes, indicated by the highest $A_W$ and $A_N$ and the lowest PD, respectively.
Consistent with the results on CUB200 (Table~\ref{tab:cub200}), CIFAR100 (Table~\ref{tab:cifar}), and \textit{mini}ImageNet (Table~\ref{tab:mini}) in the main manuscript, these results demonstrate CLOSER's ability to generalize on a more challenging dataset and beyond the CNN architecture.

%% file: supple/Table/imagenet.tex
\begin{table*}[h!]
   \centering
   \scriptsize
   \caption{\textbf{Experiments on our ImageNet-FSCIL dataset.}
   RS refers to `Representation Spreading'.
   }
   \resizebox{1\linewidth}{!}{
   \begin{tabular}{lcccccccccccccc}
   \toprule[\heavyrulewidth]
        \multicolumn{1}{c}{\multirow{2}{*}{Method}} & \multicolumn{11}{c}{Acc. in each session (\%) } &
        \multicolumn{1}{c}{\multirow{2}{*}{PD (\%) $\downarrow$ }} & \multicolumn{1}{c}{\multirow{2}{*}{$A_B$ (\%)}} & \multicolumn{1}{c}{\multirow{2}{*}{$A_N$ (\%)}}\\
        \cmidrule{2-12}
        \multicolumn{1}{c}{} & 0 & 1 & 2 & 3 & 4 & 5 & 6 & 7 & 8 & 9 & 10\\
   \toprule
   Baseline & 73.68 & 69.81 & 66.53 & 63.73 & 60.97 & 58.65 & 56.49 & 54.38 & 52.71 & 51.03 & 49.48 & 24.20 & \textbf{71.14} & 16.99\\
   Baseline+RS & 69.13 & 65.73 & 62.72 & 60.41 & 58.01 & 56.09 & 54.15 & 52.14 & 50.78 & 49.42 & 48.02 & 21.11 & 67.16 & 19.32\\
   \textbf{CLOSER}& 71.37 & 68.02 & 65.05 & 62.64 & 60.21 & 58.17 & 56.30 & 54.29 & 52.87 & 51.62 & \textbf{50.28} & \textbf{21.09} & 69.38 & \textbf{21.63}\\
   \bottomrule[\heavyrulewidth]
   \end{tabular}
   }
   \label{tab:imagenet}
\end{table*}

%% file: supple/Table/vit.tex
\begin{table}[h!]
\centering
\scriptsize
\caption{\textbf{CUB200 experiments with CNN and ViT network.}
RS refers to `Representation Spreading'.
Both the ResNet-18 and ViT-B/16 are pre-trained on ImageNet.
All results are obtained after all incremental sessions.
Please refer to Section~\ref{subsec:exp_details} for the other details on CUB200 experiments.
}
\resizebox{1\linewidth}{!}{
    \begin{tabular}{lccccccccccccccc}
   \toprule[\heavyrulewidth]
        \multicolumn{1}{c}{\multirow{2}{*}{Architecture}} & \multicolumn{1}{c}{\multirow{2}{*}{Method}} & \multicolumn{11}{c}{Acc. in each session (\%) } &
        \multicolumn{1}{c}{\multirow{2}{*}{PD (\%) $\downarrow$ }} & \multicolumn{1}{c}{\multirow{2}{*}{$A_B$ (\%)}} & \multicolumn{1}{c}{\multirow{2}{*}{$A_N$ (\%)}}\\
        \cmidrule{3-13}
        & & 0 & 1 & 2 & 3 & 4 & 5 & 6 & 7 & 8 & 9 & 10\\
   \toprule
   \multirow{3}{*}{ResNet-18}& Baseline & 79.71 & 76.23 & 73.35 & 69.61 & 68.22 & 66.23 & 65.32 & 64.39 & 62.52 & 62.43 & 61.43 & 18.28 & 76.15 & 47.03\\
   &Baseline+RS & 77.44 & 74.26 & 71.49 & 68.24 & 67.16 & 64.96 & 64.33 & 63.71 & 62.03 & 61.90 & 61.29 & 16.15 & 74.76 & 48.12\\
   &\textbf{CLOSER}& 79.34 & 75.92 & 73.50 & 70.47 & 69.24 & 67.22 & 66.73 & 65.69 & 64.00 & 64.02 & \textbf{63.58} & \textbf{15.76} & \textbf{76.40} & \textbf{51.06}\\
   \midrule
   \multirow{3}{*}{Vit-B/16}& Baseline & 82.65 & 79.86 & 77.78 & 75.03 & 73.98 & 72.19 & 71.02 & 70.64 & 68.77 & 69.11 & 68.81 & 13.84 & 80.80 & 57.10\\
   &Baseline+RS & 82.09 & 79.32 & 77.58 & 75.40 & 74.38 & 72.33 & 71.35 & 70.80 & 69.17 & 69.57 & 69.45 & 12.64 & 80.20 & 58.94\\
   &\textbf{CLOSER}& 83.38 & 81.01 & 79.50 & 77.28 & 76.49 & 74.78 & 73.97 & 73.24 & 71.51 & 71.90 & \textbf{71.71} & \textbf{11.67} & \textbf{81.32} & \textbf{62.32}\\
   \bottomrule[\heavyrulewidth]
   \end{tabular}
   }
    \label{tab:vit}
\end{table}

%% file: supple/comparison.tex
\input{supple/Table/analysis}
\section{Comparison with prior works on quality of learned representation}\label{sec:supple:comparison}
In addition to the comparisons in \Cref{tab:cub200,tab:cifar,tab:mini}, we compare against the mentioned previous representation-learning-based few-shot class incremental learning works~\cite{zou2022margin,song2023learning}, with respect to the quality of learned representations.
Specifically, we quantify the quality of representations using $\mathcal{T}(f_{\boldsymbol{\theta}})$ (defined in Eq.~\eqref{eq:metric} of the main manuscript) and the accuracy on new ($A_N$), base ($A_B$), and whole classes ($A_W$).
We measure $\mathcal{T}(f_{\boldsymbol{\theta}})$ to quantify how the representations of new classes are distinguishable from those of base classes.
The results presented in Table~\ref{tab:supp:trans} indicate that the representations obtained by previous works display relatively high $A_B$ but relatively low values for $\mathcal{T}(f_{\boldsymbol{\theta}})$ and $A_N$, indicating their lack of transferability of learned representation.
By contrast, our CLOSER achieves the comparable $A_B$ and the highest $\mathcal{T}(f_{\boldsymbol{\theta}})$ and $A_N$, indicating that CLOSER can yield representations with better trade-off between discriminability on base classes and transferability to new classes.
Although the prior works attempt to improve the trade-off between discriminability and transferability of the learned representation, they still rely on enlarging inter-class distance~\cite{song2023learning} or the spread of representation via negative class margin~\cite{zou2022margin}.
Thus, these results indicate that \textit{reducing} inter-class distance is significantly effective for striking a better balance between discriminability and transferability.

%% file: supple/Table/analysis.tex
\begin{table}[t!]
\centering
\scriptsize
\caption{\textbf{Comparison with prior works on the quality of learned representation.}
    We obtained the results of the previous works using the officially released codes.
    For all experiments, we use ResNet-20 and ResNet-18 for CIFAR100 and \textit{mini}ImageNet, respectively, as a backbone model.
    }
\resizebox{0.7\linewidth}{!}{
    \begin{tabular}{ccccccccc}
      \toprule
      \multirow{2}{*}{Method} & \multicolumn{4}{c}{CIFAR100} & \multicolumn{4}{c}{\textit{mini}ImageNet}\\
      \cmidrule{2-9}
      & $\mathcal{T}(f_{\boldsymbol{\theta}})$ & $A_N$(\%) & $A_B$(\%) & $A_W$(\%) & $\mathcal{T}(f_{\boldsymbol{\theta}})$ & $A_N$(\%) & $A_B$(\%) & $A_W$(\%)\\
      \toprule
      CLOM~\cite{zou2022margin} & 0.66 & 18.95 & 70.56 & 49.92 & 0.68 & 11.70 & 72.20 & 48.00\\
      SAVC~\cite{song2023learning} & 0.55 & 17.18 & 70.43 & 49.53 & 0.74 & 18.50 & \textbf{72.36}& 50.82\\
      \textbf{CLOSER (ours)} & \textbf{0.77} & \textbf{27.23} & \textbf{70.72} & \textbf{53.32} & \textbf{0.86} & \textbf{25.28} & 72.03 & \textbf{53.33}\\
      \bottomrule
    \end{tabular}}
    \label{tab:supp:trans}
\end{table}

%% file: supple/margin_temp.tex
\section{Analysis on Class Margin}\label{sec:supp:margintemp}
In this section, we compare the effects of the class margin parameter ($m$) and the temperature parameter ($\tau$) in the softmax cross-entropy loss on representation learning in the context of FSCIL.
The results in Fig.~\ref{fig:supp:margintemp_plot} show the accuracy on the whole ($A_W$), base ($A_B$), and new classes ($A_N$) at the end of all training sessions with varying margin and temperature values.
As noted in the previous works~\cite{kornblith2021why,liu2020negative,zou2022margin}, when the margin and temperature decrease, $A_N$ tends to increase, while $A_B$ tends to decrease (except the case when $m=-0.2$ and $\tau=1/32$ due to unstable training).
However, we find that the impact of the margin becomes marginal when the temperature is high.
For example, when $\tau=1/8$, the difference between the highest and lowest $A_N$ is roughly 3$\%$, a relatively minor variation compared to the approximately 9$\%$ observed with a lower temperature setting.
The results with respect to $A_W$ also show that the trade-off between $A_B$ and $A_N$ has a relatively higher correlation with temperature than with margin, and the highest $A_W$ is achieved when $\tau=1/32$ and $m=0$.
The feature visualization analysis depicted in Fig.~\ref{fig:supp:margintemp_visualization} also shows similar results, showing that the temperature has a greater impact on the learned representation than the margin.
In particular, we note that as $\tau$ decreases, it encourages a more dispersed representation, a valuable characteristic for enhancing transferability.
Based on this analysis, we regard the temperature parameter as a more effective tool for addressing the issue of base class overfitting and consequently improving transferability.
Moreover, we observe that a negative margin does not promote narrow inter-class separation; instead, it is more associated with representation spread, as depicted in Fig.~\ref{fig:supp:margintemp_visualization}, underscoring the difference between our work and the previous work~\cite{zou2022margin}.

\input{supple/Figure/margintemp_plot}

%% file: supple/Figure/margintemp_plot.tex
\begin{figure*}[t!]
    \centering
    \includegraphics[width=0.9\linewidth]{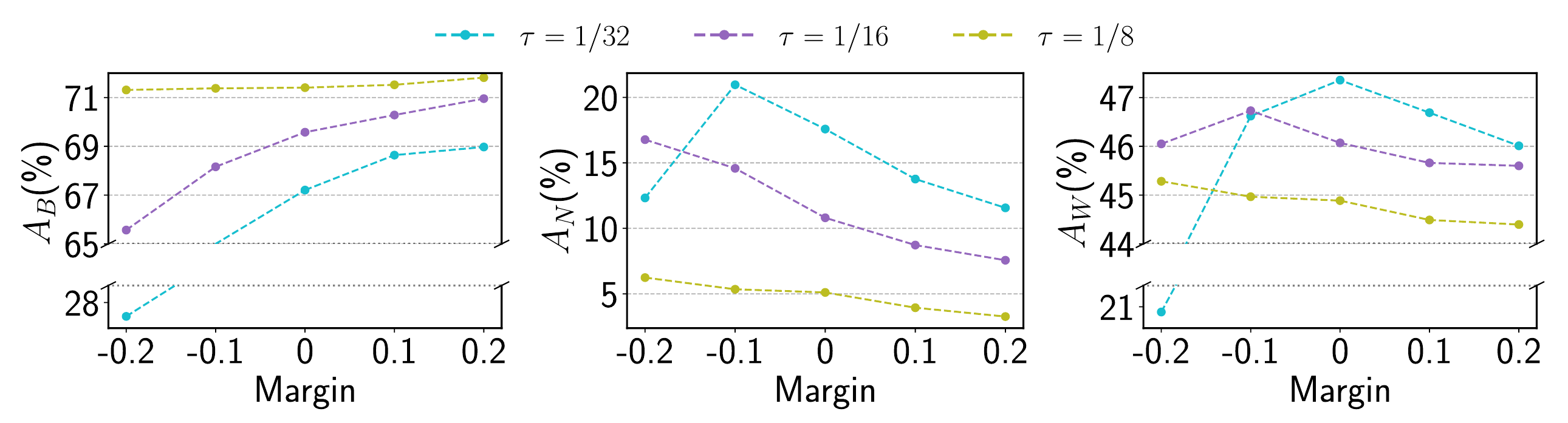}
    \caption{\textbf{Comparison of the impact of the margin and temperature parameters in the FSCIL problem.}
    Lowering temperature has a relatively greater influence on the performance than margin.
    The experiments are conducted on CIFAR100 and we report the averaged results from 3 independent experiments.}
    \label{fig:supp:margintemp_plot}
\end{figure*}

%% file: supple/details.tex
\section{Implementation Details}\label{sec:supple:details}
\noindent{\textbf{Self-Supervised Contrastive Learning.}}
For self-supervised contrastive learning (SSCL) discussed in Section~\ref{sec:spread}, we generate different views for each image in a mini-batch via data augmentation.
For both CIFAR100 and \textit{mini}ImageNet experiments, we apply random resized cropping, random horizontal flipping with probability 0.5, and random AutoAugment~\cite{cubuk2019autoaug} with probability 0.5.
For CUB200 experiments, we apply the random resized cropping and the random horizontal flipping with probability 0.5 but without AutoAugment since color information is crucial for fine-grained classification of the CUB200 dataset.
Unlike the previous methods on SSCL~\cite{chen2020simple,he2020momentum}, we do not use either a non-linear projection head or a momentum encoder since we found that the performance difference is marginal.

\noindent{\textbf{Optimization.}}
For optimization, we adhere to standard protocols from previous works~\cite{zhang2021cec}.
We use the stochastic gradient descent optimizer with weight decay of $5\cdot10^{-4}$ and Nesterov momentum 0.9.
We set the initial learning rate as 0.1, 0.1, and 0.005 for CIFAR100, \textit{mini}ImageNet, and CUB200 experiments, respectively, and decay them by 0.1 at the 80\% and 90\% of the total training epochs.
We set the total training epochs as 200 for both CIFAR100 and \textit{mini}ImageNet experiments and 50 for CUB200 experiments.

\noindent{\textbf{Hyper-parameters Search Strategy.}}\label{supp:sec:details}
In the proposed method, there are three hyperparameters including the temperature parameter $\tau$ in the softmax function and the loss weights for the SSCL ($\lambda_{\text{ssc}}$) and the inter-class distance loss ($\lambda_{\text{inter}}$).
Since we cannot acquire the validation dataset for new classes in the current benchmark setting, we perform a hyper-parameter search strategy using a synthetic dataset for new classes.
Following CEC~\cite{zhang2021cec}, we synthesize a new class by rotating images of a base class with a certain rotation degree.
After the base session, we conduct fake incremental sessions using a few synthetic new class samples and measure the overall performance using the validation set of the base and the fake new classes.
We split the dataset for base classes to obtain the validation set for the base classes.
We observe that this validation strategy provides a confident measure of the actual test performance of our algorithm, enabling an effective hyperparameters search.

\noindent{\textbf{Mutual Information Estimation.}}
To evaluate the mutual information, $I(X;Z)$ and $I(Y;Z)$, we adopt MINE~\cite{belghazi18mutual} method.
Specifically, we train a 4-layer Multi-Layer Perceptron (MLP) with ReLU activation to estimate the mutual information.
For $I(X;Z)$, we set the hidden dimension of the estimator as 256 and the input dimension as ($32 \times 32 \times 3 + 64$), which is the summation of the shape of a flattened image and latent representation.
For $I(Y;Z)$, we set the hidden dimension as 32 and the input dimension as ($C + 64$), which is the summation of the number of classes (base, new, or whole classes) and the dimension of latent representation.
For optimization, we adopt the Adam optimizer and set the learning rate to 1e-4.
We train the estimator for 10K iterations.

%% file: supple/Figure/margintemp_visualizations.tex
\begin{figure*}[t!]
    \centering
    \begin{subfigure}[b]{0.77\linewidth}
    \centering
    \includegraphics[width=\textwidth]{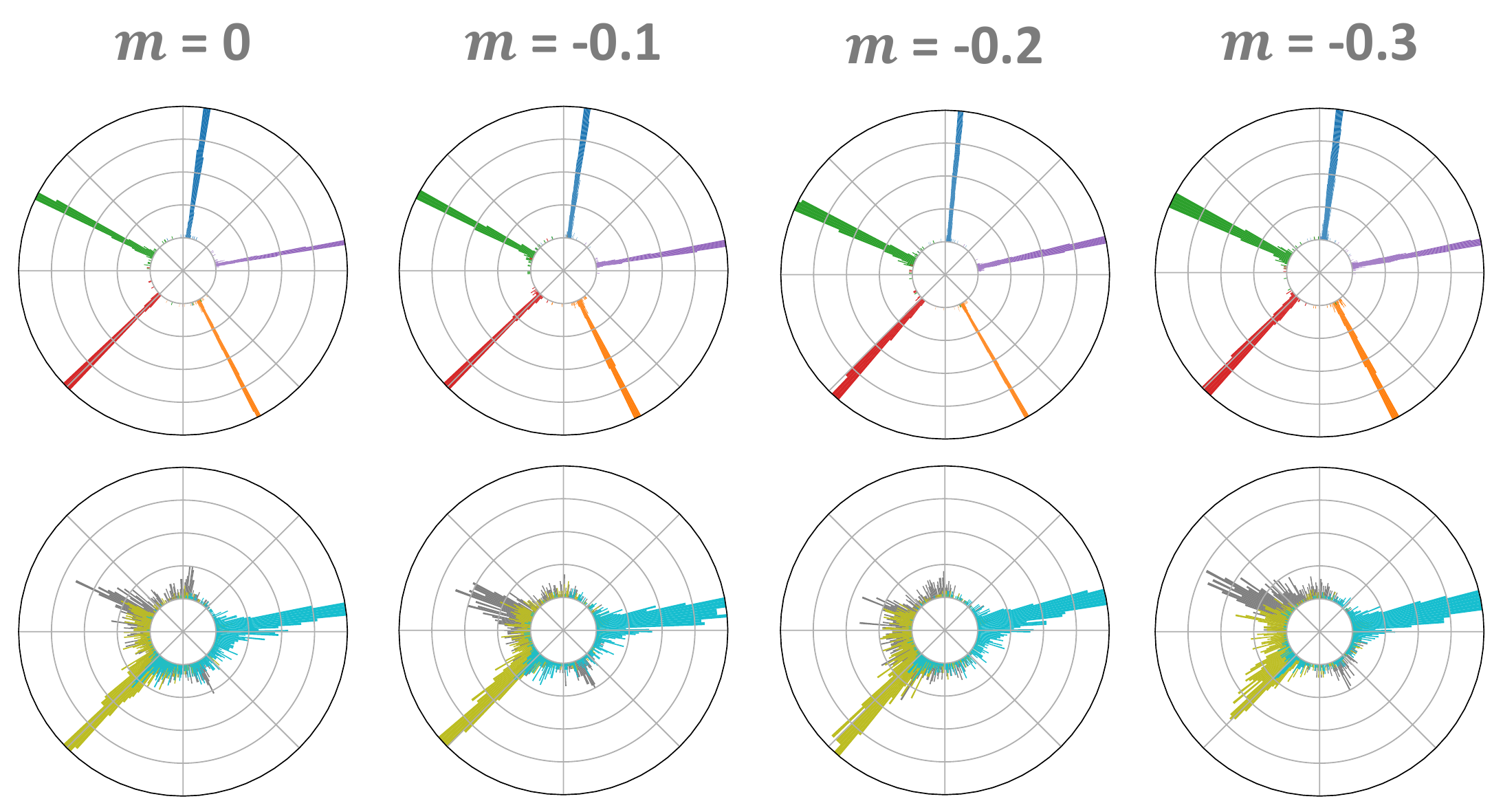}
    \caption{$\tau=\frac{1}{4}$}
    \end{subfigure}

    \begin{subfigure}[b]{0.77\linewidth}
    \centering
    \includegraphics[width=\textwidth]{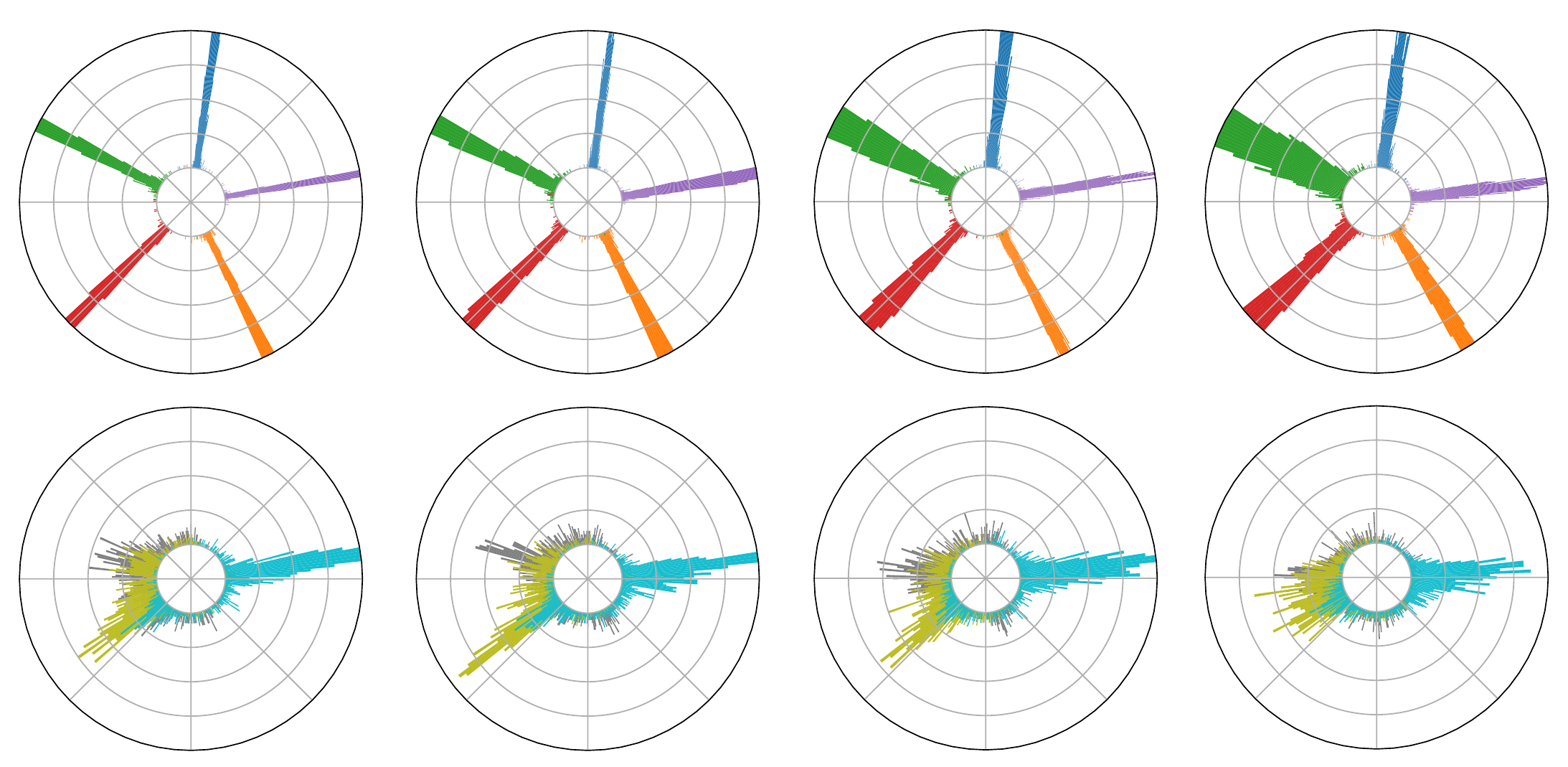}
    \caption{$\tau=\frac{1}{8}$}
    \end{subfigure}
    
    \begin{subfigure}[b]{0.77\linewidth}
    \centering
    \includegraphics[width=\textwidth]{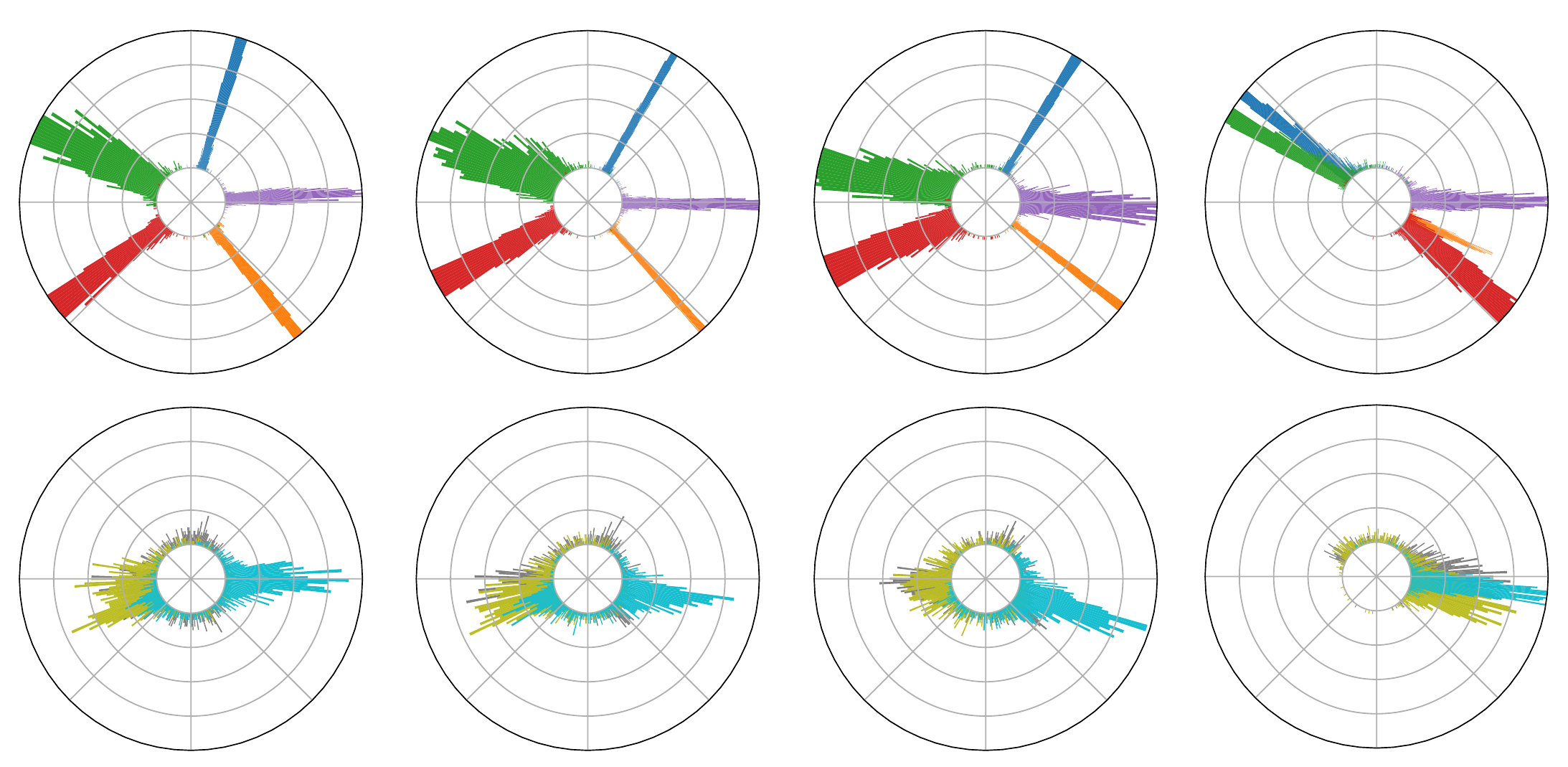}
    \caption{$\tau=\frac{1}{16}$}
    \end{subfigure}
    \caption{\textbf{Visualization of representation for comparison of the effect of margin and temperature.}
    Lowering temperature has a relatively greater influence on the performance than margin.
    We train a network on MNIST dataset with a 2-dimensional feature space and visualize angular histograms without a dimension reduction.
    The first and second row in each subfigure indicates the results on base and new classes, respectively.
    Each color represents a different class.
    }
    \label{fig:supp:margintemp_visualization}
\end{figure*}